\documentclass[acmsmall]{acmart}

\usepackage[english]{babel}

\usepackage{amsmath,amsfonts}
\usepackage{booktabs}       
\usepackage{nicefrac}       
\usepackage{microtype}      
\usepackage{xcolor}
\usepackage[T1]{fontenc}    
\usepackage{amsthm}
\usepackage{algorithm}
\usepackage{textcomp}
\usepackage{stfloats}
\usepackage{url}
\usepackage{verbatim}
\usepackage{graphicx}
\usepackage{bm}
\usepackage{cleveref}
\usepackage{enumitem}
\usepackage{mathtools}
\theoremstyle{plain}
\theoremstyle{definition}

\usepackage{array}
\usepackage{lipsum}
\usepackage{multirow}
\usepackage{makecell}
\usepackage{xcolor}
\usepackage{tikz}
\usepackage{ulem}
\usepackage{etoolbox}
\newcommand*{\circled}[1]{\lower.7ex\hbox{\tikz\draw (0pt, 0pt)%
    circle (.5em) node {\makebox[1em][c]{\small #1}};}}
\robustify{\circled}

\AtBeginDocument{%
  \providecommand\BibTeX{{%
    \normalfont B\kern-0.5em{\scshape i\kern-0.25em b}\kern-0.8em\TeX}}}

\setcopyright{acmcopyright}
\copyrightyear{2022}
\acmYear{2022}

\usepackage{microtype}
\usepackage[T1]{fontenc}    
\usepackage{url}            
\usepackage{booktabs}       
\usepackage{nicefrac}       
\usepackage{graphicx}
\usepackage{subfigure} 
\usepackage{bm}

\usepackage{cleveref}
\usepackage{enumitem}
\usepackage{mathtools}
\usepackage{array}
\usepackage{lipsum}
\usepackage{hyperref}
\usepackage{multirow}
\usepackage{makecell}
\usepackage{xcolor}
\usepackage{tikz}
\usepackage{etoolbox}
\robustify{\circled}

\usepackage{afterpage}
\usepackage{makecell}
\usepackage{tabulary}
\usepackage{xcolor}
\usepackage{colortbl}
\usepackage{prettyref}
\usepackage{framed}
\usepackage{booktabs}       
\newcommand{\pref}[1]{\prettyref{#1}}

\newcommand{\savehyperref}[2]{\texorpdfstring{\hyperref[#1]{#2}}{#2}}
\newrefformat{eq}{\savehyperref{#1}{Eq. \textup{(\ref*{#1})}}}
\newrefformat{eqn}{\savehyperref{#1}{Eq.~\textup{(\ref*{#1})}}}
\newrefformat{lem}{\savehyperref{#1}{Lemma~\ref*{#1}}}
\newrefformat{assump}{\savehyperref{#1}{Assumption~\ref*{#1}}}
\newrefformat{defi}{\savehyperref{#1}{Definition~\ref*{#1}}}\newrefformat{tab}{\savehyperref{#1}{Table~\ref*{#1}}}
\newrefformat{lemma}{\savehyperref{#1}{Lemma~\ref*{#1}}}
\newrefformat{line}{\savehyperref{#1}{Line~\ref*{#1}}}
\newrefformat{thm}{\savehyperref{#1}{Theorem~\ref*{#1}}}
\newrefformat{corr}{\savehyperref{#1}{Corollary~\ref*{#1}}}
\newrefformat{cor}{\savehyperref{#1}{Corollary~\ref*{#1}}}
\newrefformat{sec}{\savehyperref{#1}{Section~\ref*{#1}}}
\newrefformat{app}{\savehyperref{#1}{Appendix~\ref*{#1}}}
\newrefformat{ex}{\savehyperref{#1}{Example~\ref*{#1}}}
\newrefformat{fig}{\savehyperref{#1}{Figure~\ref*{#1}}}
\newrefformat{alg}{\savehyperref{#1}{Algorithm~\ref*{#1}}}
\newrefformat{rem}{\savehyperref{#1}{Remark~\ref*{#1}}}
\newrefformat{conj}{\savehyperref{#1}{Conjecture~\ref*{#1}}}
\newrefformat{prop}{\savehyperref{#1}{Proposition~\ref*{#1}}}
\newrefformat{proto}{\savehyperref{#1}{Protocol~\ref*{#1}}}
\newrefformat{prob}{\savehyperref{#1}{Problem~\ref*{#1}}}
\newrefformat{claim}{\savehyperref{#1}{Claim~\ref*{#1}}}
\newrefformat{que}{\savehyperref{#1}{Question~\ref*{#1}}}
\newrefformat{op}{\savehyperref{#1}{Open Problem~\ref*{#1}}}
\newrefformat{fn}{\savehyperref{#1}{Footnote~\ref*{#1}}}
\newrefformat{property}{\savehyperref{#1}{Property~\ref*{#1}}}

\usepackage{xspace}
\usepackage{amsmath,amsfonts}
\usepackage{amsthm}
\usepackage{amssymb}
\usepackage{bbm}
\usepackage{graphicx}
\usepackage{multicol}
\usepackage{algorithm}
\usepackage[noend]{algpseudocode}
\usepackage{braket}
\usepackage{array}
\usepackage{bm}
\usepackage{makecell}
\usepackage{tabulary}

\newcommand{\calF}{\mathcal{F}}
\newcommand{\calC}{\mathcal{C}}
\newcommand{\calD}{\mathcal{S}}
\newcommand{\calRD}{\mathcal{D}}
\newcommand{\calL}{\mathcal{L}}
\newcommand{\calN}{\mathcal{N}}

\newcommand{\calA}{\mathcal{A}}

\newcommand{\calM}{\mathcal{M}}

\newcommand{\calW}{\mathcal{W}}

\newcommand{\calO}{\mathcal{B}}
\newcommand{\calRO}{\mathcal{O}}

\newcommand{\R}{\mathbb{R}}

\newcommand{\E}{\mathbb{E}}

\newcommand{\one}{\mathbbm{1}}

\newcommand{\red}[1]{{\color{red}#1}}
\newcommand{\blue}[1]{{\color{black}#1}}

\newcommand{\gray}[1]{{\color{gray}#1}}

\newtheorem*{theorem*}{Theorem}

\newtheorem{assumption}{Assumption}[section]

\newtheorem*{lemma*}{Lemma}

\theoremstyle{definition}

\newtheorem{thm}{Theorem}[section]

\newtheorem{lem}[thm]{Lemma}

\begin{document}

\title{A Meta-learning Framework for Tuning Parameters of Protection Mechanisms in Trustworthy Federated Learning}


%
\author{Xiaojin Zhang}
\email{xiaojinzhang@hust.edu.cn}
\affiliation{%
  \institution{Huazhong University of Science and Technology}
  \streetaddress{Wuhan}
  \country{China}
}

\author{Yan Kang}
\email{yangkang@webank.com}
\affiliation{%
  \institution{Webank}
  \city{Shenzhen}
  \country{China}
}

\author{Lixin Fan}
\email{lixinfan@webank.com}
\affiliation{%
  \institution{Webank}
  \city{Shenzhen}
  \country{China}
}


\author{Kai Chen}
\email{kaichen@cse.ust.hk}
\affiliation{%
  \institution{Hong Kong University of Science and Technology}
  \streetaddress{Clear Water Bay}
  \country{China}
}


\author{Qiang Yang}
\email{qyang@cse.ust.hk}
\authornote{Corresponding author}
\affiliation{%
  \institution{WeBank and Hong Kong University of Science and Technology}
  \country{China}
}






%
\renewcommand{\shortauthors}{Trovato and Tobin, et al.}

\begin{abstract} 

Trustworthy Federated Learning (TFL) typically leverages protection mechanisms to guarantee privacy. However, protection mechanisms inevitably introduce utility loss or efficiency reduction while protecting data privacy. Therefore, protection mechanisms and their parameters should be carefully chosen to strike an optimal trade-off between \textit{privacy leakage}, \textit{utility loss}, and \textit{efficiency reduction}. To this end, federated learning practitioners need tools to measure the three factors and optimize the trade-off between them to choose the protection mechanism that is most appropriate to the application at hand. Motivated by this requirement, we propose a framework that (1) formulates TFL as a problem of finding a protection mechanism to optimize the trade-off between privacy leakage, utility loss, and efficiency reduction and (2) formally defines bounded measurements of the three factors. We then propose a meta-learning algorithm to approximate this optimization problem and find optimal protection parameters for representative protection mechanisms, including Randomization, Homomorphic Encryption, Secret Sharing, and Compression. We further design estimation algorithms to quantify these found optimal protection parameters in a practical horizontal federated learning setting and provide a theoretical analysis of the estimation error.
\end{abstract}



\maketitle

\section{Introduction}


With the enforcement of data privacy regulations such as the General Data Protection Regulation (GDPR)\footnote{GDPR is applicable as of May 25th, 2018, in all European member states to harmonize data privacy laws across Europe. \url{https://gdpr.eu/}}, sharing data owned by a company with others is prohibited. To mine values of data dispersed among multiple entities (e.g., edge devices or organizations) while respecting data privacy, federated learning (FL)~\cite{mcmahan2016federated,mcmahan2017communication,konevcny2016federated,konevcny2016federated_new} that enables multiple parties to train machine learning models collaboratively without sharing private data has drawn lots of attention in recent years.


Intuitively, FL guarantees better privacy than centralized learning, which requires gathering data in a centralized place to train models. However, recent works demonstrate that adversaries in FL are able to reconstruct private data via privacy attacks, such as DLG \cite{zhu2020deep}, Inverting Gradients \cite{geiping2020inverting}, Improved DLG \cite{zhao2020idlg}, GradInversion \cite{yin2021see}. These attacks infer private information from the Bayesian perspective and hence are referred to as the Bayesian inference attack~\cite{zhang2022trading}. The literature has proposed a variety of privacy protection mechanisms, such as Homomorphic Encryption (HE)~\cite{gentry2009fully,batchCryp}, Randomization~\cite{geyer2017differentially,truex2020ldp,abadi2016deep}, Secret Sharing~\cite{SecShare-Adi79,SecShare-Blakley79,bonawitz2017practical} and Compression~\cite{nori2021fast} to thwart Bayesian inference attacks. However, these protection mechanisms inevitably bring about a loss of utility or a decrease in efficiency. As a result, FL practitioners need tools to measure and trade off the three conflicting factors, namely privacy leakage, utility loss, and efficiency reduction, to choose the most appropriate protection mechanism. 

Motivated by this requirement, we formulate the problem of optimizing the trade-off between privacy leakage, utility loss, and efficiency reduction as a constrained optimization problem that finds a protection mechanism achieving minimal utility loss and efficiency reduction given a privacy budget. Then, we define measurements of the three factors and design a meta-learning algorithm to solve this optimization problem by finding the optimal
protection parameters of a given protection mechanism. Further, we propose estimation algorithms to measure optimal parameter values for four representative protection mechanisms. The main results of our research are summarized as follows (also see Fig. \ref{fig:outline} for a pictorial summary):

In terms of the privacy measurement, we use the discrepancy between the posterior belief and prior belief of the attacker on private data as the privacy measurement for the Bayesian inference attackers, which is referred to as Bayesian Privacy (BP) \cite{zhang2022no, zhang2022trading}. Differential privacy (DP), as a widely used privacy measurement, establishes a relationship between the privacy budget and the noise leveraged to protect data privacy. However, applying DP to measure privacy is difficult for popular protection mechanisms, such as the compression mechanism. The common logic of BP and DP is to evaluate the distance between the two distributions we are concerned about. The difference is that DP considers the worst-case neighboring dataset two conditional distributions, whereas BP considers the distance between prior and posterior distributions and is more suitable for the Bayesian setting. \citet{BDP-icml20} proposed a new variant of DP called Bayesian differential privacy (BDP), which considers the distribution of datasets. BDP is a relaxed form of DP, which considers data distribution to make privacy guarantees more practical. Indeed, BDP does not model prior distribution and posterior distribution, whereas BP used in this work captures the prior and posterior in the Bayesian learning framework. In sum, our main contributions include:
\begin{itemize}
    \item We provide upper bounds for \textit{utility loss}, \textit{privacy leakage} and \textit{efficiency reduction}. The bounds of these three key quantities are related to the total variation distance between the prior and posterior distributions, which can further be depicted using the parameter of the protection mechanism. The bounds for privacy leakage are depicted using two abstract quantities $C_1$ and $C_2$ (see Theorem \ref{thm: privacy_leakage_upper_bound_main_result}). We design practical algorithms leveraging DLG~\cite{zhu2020deep} attack to estimate two key quantities and provide theoretical guarantee for the estimation error (see \pref{thm: error_for_C_1} and \pref{thm: estimation_error_privacy}). As a comparison, \cite{zhang2022no, zhang2022trading} only provide lower bounds for \textit{utility loss}, \textit{privacy leakage} and \textit{efficiency reduction}, which are not sufficient to provide robust metrics constructing the optimization framework.
     \item We formalize the trade-off between privacy, utility, and efficiency as an optimization problem with respect to the protection parameter. With the derived bounds, we design a meta-learning algorithm that finds the optimal protection parameter characterizing the protection mechanism under the optimization framework (\pref{thm: meta_alg_for_opt_problem}). 
    \item We apply our proposed algorithm to find the optimal protection parameter for several representative protection mechanisms including \textit{Homomorphic Encryption}, \textit{Randomization}, \textit{Secret Sharing} and \textit{Compression}  under our proposed framework (\pref{thm: optimal_parameter_randomization}$\sim$\pref{thm: opt_parameter_compression}). We formalize the parameters of Randomization, Paillier Homomorphic Encryption, Secret Sharing, and Compression as the variance of the added random noise, length of the ciphertext, bound of the added random value, and probability of keeping the model parameters, respectively. The optimization framework, the derived bounds, and the proposed estimation algorithms together provide an avenue for evaluating the trade-off between privacy, utility, and efficiency using the parameters of the protection mechanism, which guides the selection of protection parameters to adapt to specific requirements.
\end{itemize}
\begin{figure}[h]
\centering
\includegraphics[width = 0.95\columnwidth]{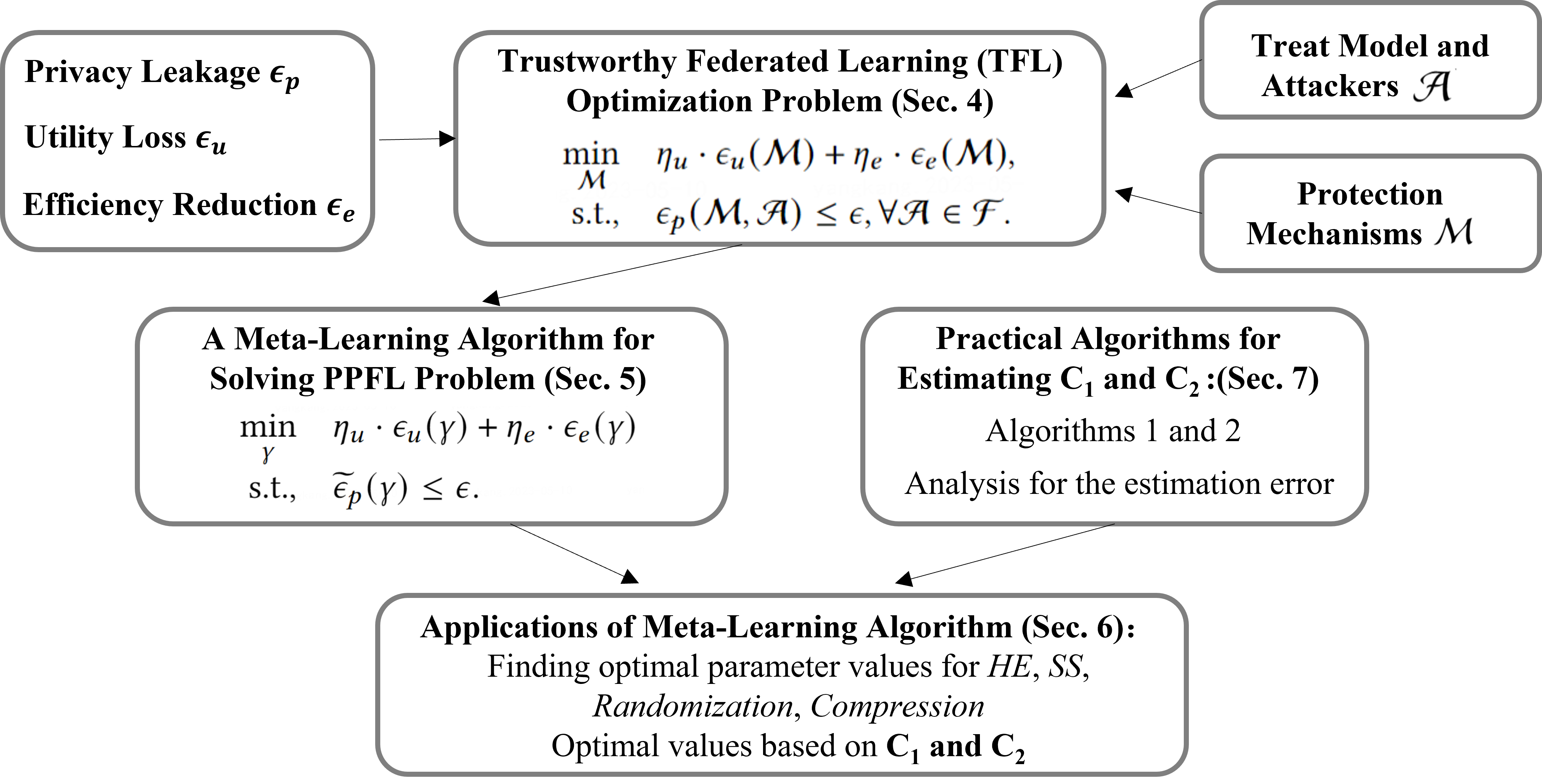}
\caption{Outline of this work.} \label{fig:outline}
\end{figure}

\section{Related Work}

We brief review related works concerning privacy measurements, privacy attacking and protection mechanisms, and privacy-utility trade-off.


\subsection{Privacy Measurements}
\label{sec:related:DP}

Differential privacy (DP) is a mathematically sound definition of privacy introduced by \citet{dwork2006calibrating}. \citet{abadi2016deep} proposed novel techniques to analyze differential privacy for training deep neural networks. Their techniques can be applied to various optimization approaches, including Momentum \cite{rumelhart1986learning} and AdaGrad \cite{duchi2011adaptive}. \citet{BDP-icml20} proposed a variant of DP called Bayesian differential privacy (BDP). It considers the distribution of the dataset. A privacy accounting approach was further proposed. \citet{rassouli2019optimal} measured privacy using average total variation distance and utility using mutual information, minimum mean-square error (MMSE), and probability of error separately. It is challenging to achieve pure DP in many realistic learning environments. \citet{eilat2021bayesian} measured the loss of privacy using the gap between the prior and posterior beliefs of the designer towards the type of an agent. The gap was calculated using Kullback-Leibler (KL) divergence. \citet{foulds2016theory} pointed out the limitations of using posterior sampling to obtain privacy in practice. The advantage of the Laplace mechanism for Bayesian inference is illustrated from both theoretical and empirical points of view.

Some works consider privacy measurements similar to our privacy measurements. \citet{du2012privacy} proposed maximum and average information leakage, designed for scenarios that assess the privacy threat in light of some utility restrictions. The privacy leakage is measured using the KL divergence by \citet{gu2021federated}. In contrast to the widely-used KL divergence, we employ JS divergence to measure privacy leakage following \cite{zhang2022no,zhang2022trading}. The advantage of Jensen–Shannon (JS) divergence over KL divergence is that it is symmetrical, and its square root satisfies the triangle inequality \blue{\cite{nielsen2019jensen, endres2003new}}. The fulfillment of triangle inequality facilitates the quantification of the trade-offs.

\subsection{Privacy Attack and Protection Mechanisms}
\label{sec:related:prot}

Federated learning enables multiple/many participants to train machine learning models leveraging but without sharing participants' private data. However, Deep Leakage from Gradient (DLG)~\cite{zhu2019dlg} and Model Inversion (MI) attacks~\cite{fredrikson2015model} and their follow-up works~\cite{geiping2020inverting,zhao2020idlg,yin2021see, fredrikson2015model, gu2021federated} demonstrated that adversaries are able to reconstruct private data by solving an optimization problem based on shared model or gradients. 


Various protection mechanisms have been adopted for privacy-preserving federated learning to protect private data from being inferred by adversaries. The most widely used ones are \textit{Homomorphic Encryption (HE)}~\cite{gentry2009fully,batchCryp}, \textit{Randomization Mechanism}~\cite{geyer2017differentially,truex2020ldp,abadi2016deep}, \textit{Secret Sharing}~\cite{SecShare-Adi79,SecShare-Blakley79,bonawitz2017practical} and \textit{Compression Mechanism} \cite{nori2021fast}. The \textit{Randomization Mechanism} refers to adding Laplace or Gaussian noise to released statistics for reducing the dependence between the released statistics (e.g., model parameters and gradients) and the private data.
Another school of FL~\cite{gupta2018distributed,gu2021federated} tries to protect privacy by splitting a neural network into private and public models and sharing only the public one~\cite{kang2021privacy,gu2021federated}.

\subsection{Trade-off Between Privacy, Utility, and Efficiency}

\citet{dwork2014algorithmic,duchi2013local,du2012privacy} quantified the privacy-utility trade-off when privacy was measured using DP, LDP, and IP, respectively. \citet{du2012privacy} introduced a privacy metric called \textit{maximum information leakage}, which is similar to the definition of \textit{maximum Bayesian Privacy}. The optimization problem they consider aims at minimizing \textit{maximum information leakage}, subject to utility constraints, which is measured using the expected distance between the original and distorted information. As a comparison, the utility loss in our work measures the model performance associated with the original model information and the distorted model information, which directly depicts the model's utility.

\section{Preliminaries}\label{sec:framework}

In this work, we focus on Horizontal Federated Learning (HFL), which involves a total of $K$ clients. Let $D_k$ denote the dataset owned by client $k$. The objective of the $K$ clients is to train a global model collaboratively, and is formulated as follows:
    \begin{align*}
        W^* &= \arg\min_{W}\sum_{k = 1}^K \frac{n_k}{n}\calL_k(W, D_k),
    \end{align*} 
where $n_k$ denotes the size of the dataset $D_k$, $n=\sum_{k=1}^K n_k$, and $\calL_k(W, D_k)$ is the loss of predictions made by the model parameter $W$ on dataset $D_k$.

\subsection{Notations}
We follow the convention of representing the random variables using uppercase letters, such as $D$, and representing the particular values of random variables using lowercase letters. We denote $[K] = \{1,2,\cdots, K\}$. The distributions are represented using uppercase letters such as $F$ and $P$, and the probability density functions are represented using lowercase letters such as $f$ and $p$. We use $f_{D_k}(d)$ to represent the value of the probability density function $f$ at $D_k$, and the subindex represents the random variable. We represent the conditional density function by using the notation $f_{D_k|W_k}(d|w)$. For continuous distributions $P$ and $Q$ over $\mathbb{R}^n$, let $p$ and $q$ represent the probability densities of $P$ and $Q$, the Kullback-Leibler divergence is defined as $\text{KL}(P||Q) = \int p(x)\log (p(x)/q(x))dx$. As a smoothed version of the Kullback-Leibler divergence, the Jensen-Shannon divergence is defined as $\text{JS}(P||Q)  = \frac{1}{2}\left[\text{KL}\left(P, M\right) + \text{KL}\left(Q, M\right)\right]$, where $M = (P + Q)/2$. Let $\text{TV}(P||Q)$ represent the total variation distance between $P$ and $Q$, which is defined as $\text{TV}(P||Q) = \sup_{A\subset\mathbb R^n} |P(A) - Q(A)|$. 
The detailed descriptions of notations are illustrated in Table \ref{table: notation}.


\subsection{Threat Model}

In this section, we discuss the threat model, including the attacker's objective, capability, knowledge, and attacking methods. Let $D_k$ denote the private dataset of participant $k$ and $\Tilde{D}_k$ denote the participant $k$'s data recovered by the attacker.

\noindent\textbf{Attacker's objective}. We consider the server as the attacker who aims to infer the private data $D_k$ of participant $k$ with high fidelity upon observing the exposed information of participant $k$ from the \textit{Bayesian perspective}. In other words, the attacker's goal is to find $\Tilde{D}_k$ that maximizes its posterior belief on $D_k$ based on the exposed information of participant $k$.

\noindent\textbf{Attacker's capability}. We consider the attacker as \textit{semi-honest} such that the attacker faithfully follows the horizontal federated training protocol, but it may mount privacy attacks to infer the private data of participating parties. The attacker infers private information based on the observation of a single round.

\noindent\textbf{Attacker's knowledge}. In HFL, all participating parties send their models to the server. Therefore, the server knows each participant's model parameters and structure. The server (i.e., attacker) knows the approach for generating model gradients from the private data;

\noindent\textbf{Attacking method.}
Given the attacker's capability and knowledge, the attacker leverages \textit{deep Leakage from Gradient}~\cite{zhu2020deep} to find a dataset $\Tilde{D}_k$ that minimizes the distance between the observed model gradients and estimated model gradients.




\subsection{Protection Mechanisms}


Each client $k$ leverages a \textit{protection mechanism} $\calM: \R^{m}\rightarrow \R^{m}$ to map its plain-text model information $W_k^{\calRO}$ following a distribution $P_k^{\calRO}$ to its protected counterpart $W^{\calD}_k$ following a distribution $P_k^{\calD}$. The protected model information $W^{\calD}_k$ will be shared with the server for further training. Specifically, the protection mechanism $\calM$ protects private data $D_k$ of client $k$ by distorting model information $W_k^{\calRO}$ such that the dependency between distorted (i.e., protected) model information $W_k^{\calD}$ and $D_k$ is reduced, as compared to the dependency between the unprotected information $W_k^{\calRO}$ and $D_k$. In this way, it is challenging for the adversaries to infer $D_k$ based on $W_k^{\calD}$. The distribution $P_k^{\calD}$ is referred to as the \textit{protected distribution} of client $k$. 


Significant research has been devoted to the framework surrounding privacy attacks and defenses, as evidenced by studies such as those by Zhang et al. (2023) \cite{zhang2023towards}, Du et al. (2012) \cite{du2012privacy}, and others \cite{zhang2023game, kang2022framework, zhang2023theoretically, asi2023robustness, zhang2023probably}. We take the randomization mechanism as an example. We assume that $W_k^{\calRO}\sim P_k^{\calRO} = \calN(\mu_0,\Sigma_0)$ and the random noise $\epsilon_k\sim \calN(0,\Sigma_{\epsilon})$, where $\Sigma_0 = \text{diag}(\sigma_{1}^2,\cdots, \sigma_{m}^2)$ and $\Sigma_\epsilon = \text{diag}(\sigma_\epsilon^2, \cdots, \sigma_\epsilon^2)$. Then, the original parameter is protected by adding noise: $W_k^{\calD} = W_k^{\calRO} + \epsilon_k\sim\calN(\mu_0, \Sigma_0+ \Sigma_\epsilon)$ with distribution $P_k^{\calD} = \calN(\mu_0, \Sigma_0+ \Sigma_\epsilon)$.



\section{A Framework of Trustworthy Federated Learning}

Following the statement made in the work~\cite{kang2023cmofl} that a trustworthy federated learning system should at least optimize the trade-off between privacy, utility, and efficiency, we formulate trustworthy federated learning as a constrained optimization problem that aims to find a protection mechanism to achieve minimal utility loss and efficiency reduction given a privacy budget under a privacy attack. We then formally define utility loss, efficiency reduction, and privacy leakage in the context where protectors and attackers compete.


\subsection{Formulation of Trustworthy Federated Learning}





In trustworthy federated learning, each participant (i.e., protector) $k$ applies a protection mechanism $\calM$ to convert its plain-text model $W^{\calRO}_k$ to a protected one $W^{\calD}_k$, which will be exposed to the adversary, aiming to mitigate the privacy leakage on $D_k$. On the other hand, the adversary launches a privacy attack on the protected model information $W^{\calD}_k$ to infer as much information on $D_k$ as possible. Since the protection mechanism may jeopardize model utility and reduce efficiency, the objective of participants is to control the strength of the applied protection mechanism $\calM$ to achieve minimal utility loss and efficiency reduction while preserving privacy.
 
Let $\epsilon_u$, $\epsilon_e$, and $\epsilon_p$ denote the functions measuring utility loss, efficiency reduction, and privacy leakage, respectively. $\epsilon_p$ is impacted by both the protection mechanism $\calM$ adopted by the protector and the attacking mechanism $\calA$ exploited by the attacker, while $\epsilon_u$ and $\epsilon_e$ are impacted by $\calM$. We formulate the objective of FL participants as a constrained optimization problem that finds a protection mechanism $\calM$ achieving the minimal utility loss and efficiency reduction given a privacy budget $\epsilon$:
\begin{align} \label{eq: constraint_optimization_problem}
\begin{array}{r@{\quad}l@{}l@{\quad}l}
\quad\min\limits_{\calM}& \eta_{u}\cdot\epsilon_{u}(\calM) +  \eta_{e}\cdot\epsilon_{e}(\calM),\\
\text{s.t.,} & \epsilon_{p} (\calM, \calA)\le\epsilon, \forall\calA\in\calF.\\
\end{array}
\end{align}
where $\eta_{u}$ and $\eta_{e}$ represent the preferences of the TFL toward model utility and efficiency, respectively; $\epsilon$ is the privacy budget. 




The protection mechanism impacts privacy leakage, utility loss, and efficiency reduction. Therefore, we must measure the three factors to find the optimal protection parameter. To this end, we discuss how to measure privacy leakage, utility loss, and efficiency reduction in the next section. 




\subsection{Bayesian Privacy Leakage, Utility Loss and Efficiency Reduction}


In this section, we define the measurements of privacy leakage, utility loss, and efficiency reduction. In particular, the measurement of privacy leakage is derived from the bounds of Bayesian privacy leakage.

\paragraph{Bayesian Privacy Leakage}
The privacy leakage quantifies the discrepancy between adversaries' beliefs with and without observing leaked information leveraging the Bayesian inference attack. Therefore, we call this privacy leakage \textit{Bayesian privacy leakage}. The posterior belief is averaged in relation to the protected model information that is exposed to adversaries. The protection mechanism $\calM$ modifies the original model information $W_k^{\calRO}$ to its protected counterpart $W^{\calD}_k$ and induces the local model to behave less accurately. Consequently, the aggregated global model is less accurate, and the incurred utility loss is defined as the difference between utilities with and without protections.

We denote $F^{\calA}_k$, $F^{\mathcal O}_k$ and $F^{\calO}_k$ as the attacker's belief distribution about $D_k$ upon observing the protected information, the original (unprotected) information and without observing any information, respectively.
We assume that \blue{$F^{\calA}_k$, $F^{\mathcal O}_k$, and $F^{\calO}_k$} are continuous, and the probability density functions of which is \blue{$f^{\calA}_{D_k}$, $f^{\mathcal O}_{D_k}$}, and $f^{\calO}_{D_k}$ with respect to a measure $\mu$. Specifically, $f^{\calA}_{D_k}(d) = \int_{\mathcal{W}_k^{\calD}} f_{{D_k}|{W_k}}(d|w)dP^{\calD}_{k}(w)$, $f^{\calRO}_{D_k}(d) = \int_{\mathcal{W}_k^{\calD}} f_{{D_k}|{W_k}}(d|w)dP^{\calRO}_{k}(w)$, and $f^{\calO}_{D_k}(d) = f_{D_k}(d)$. We first define privacy leakage using the Jensen-Shannon divergence. Let $\epsilon_{p,k}$ represent the privacy leakage of client $k$. The \textit{Bayesian Privacy Leakage} of the client $k$ is defined as
\begin{align}\label{eq: def_of_pl}
&\epsilon_{p,k} = \sqrt{{\text{JS}}(F^{\calA}_k || F^{\calO}_k)},
\end{align}
where ${\text{JS}}(F^{\calA}_k || F^{\calO}_k) = \frac{1}{2}\int_{\mathcal{D}_k} f^{\calA}_{D_k}(d)\log\frac{f^{\calA}_{D_k}(d)}{f^{\calM}_{D_k}(d)}\textbf{d}\mu(d) + \frac{1}{2}\int_{\mathcal{D}_k} f^{\calO}_{D_k}(d)\log\frac{f^{\calO}_{D_k}(d)}{f^{\calM}_{D_k}(d)}\textbf{d}\mu(d)$, $F^{\calA}_k$ and $F^{\calO}_k$ represent the attacker's belief distribution about $D_k$ upon observing the protected information and without observing any information, respectively, and $f_{D_k}^{\calM}(d) = \frac{1}{2}(f^{\calA}_{D_k}(d) + f^{\calO}_{D_k}(d))$. 

Then, the Bayesian privacy leakage of the whole FL system caused by releasing the protected model information is defined as:
\begin{align*}
\epsilon_p = \frac{1}{K}\sum_{k=1}^K \epsilon_{p,k}.
\end{align*}
The privacy leakage ranges between $0$ and $1$ since the JS divergence between any two distributions is upper bounded by $1$ and lower bounded by $0$.


We utilize JS divergence instead of the commonly-used KL divergence to measure privacy leakage because it is symmetrical and satisfies the triangle inequality through its square root \cite{endres2003new}.
The following theorem provides an upper bound for privacy leakage. 
\begin{thm}\label{thm: privacy_leakage_upper_bound_main_result}
We denote $C_{1, k} = \sqrt{{\text{JS}}(F^{\calRO}_k || F^{\calO}_k)}$, $C_1 =\frac{1}{K}\sum_{k=1}^K \sqrt{{\text{JS}}(F^{\calRO}_k || F^{\calO}_k)}$, and denote $C_2 = \frac{1}{2}(e^{2\xi}-1)$, where $\xi = \max_{k\in [K]} \xi_k$, $\xi_k = \max_{w\in \mathcal{W}_k, d \in \mathcal{D}_k} \left|\log\left(\frac{f_{D_k|W_k}(d|w)}{f_{D_k}(d)}\right)\right|$ represents the maximum privacy leakage over all possible information $w$ released by client $k$, and $[K] = \{1,2,\cdots, K\}$.
Let $P_k^{\calRO}$ and $P^{\calD}_k$ represent the distribution of the parameter of client $k$ before and after being protected. Let $F^{\calA}_k$ and $F^{\calRO}_k$ represent the belief of client $k$ about $D$ after observing the protected and original parameter. The upper bound for the privacy leakage of client $k$ is 
\begin{align}
    \epsilon_{p,k}\le\max\left\{2C_{1, k} - C_2\cdot{\text{TV}}(P_k^{\calRO} || P^{\calD}_k), 3 C_2\cdot{\text{TV}}(P_k^{\calRO} || P^{\calD}_k) - C_{1, k}\right\}. 
\end{align}

Specifically, 
\begin{equation}\label{eq: upper_bound_for_privacy_of_client_k}
\epsilon_{p,k}\le\left\{
\begin{array}{cl}
2C_{1, k} - C_2\cdot{\text{TV}}(P_k^{\calRO} || P^{\calD}_k), &  C_2\cdot{\text{TV}}(P_k^{\calRO} || P^{\calD}_k)\le C_{1, k},\\
3 C_2\cdot{\text{TV}}(P_k^{\calRO} || P^{\calD}_k) - C_{1, k},  &  C_2\cdot{\text{TV}}(P_k^{\calRO} || P^{\calD}_k)\ge C_{1, k}.
\end{array} \right.
\end{equation}

The upper bound for the privacy leakage of federated learning system is 
\begin{align*}
    \epsilon_p \le \max\left\{2C_1 - C_2\cdot\frac{1}{K}\sum_{k = 1}^K {\text{TV}}(P_k^{\calRO} || P^{\calD}_k), 3 C_2\cdot\frac{1}{K}\sum_{k = 1}^K {\text{TV}}(P_k^{\calRO} || P^{\calD}_k) - C_1\right\}.
\end{align*}

Specifically, 
\begin{equation}\label{eq:two_scenarios}
\epsilon_{p}\le\left\{
\begin{array}{cl}
2C_1 - C_2\cdot\frac{1}{K}\sum_{k = 1}^K {\text{TV}}(P_k^{\calRO} || P^{\calD}_k), &  C_2\cdot\frac{1}{K}\sum_{k = 1}^K{\text{TV}}(P_k^{\calRO} || P^{\calD}_k)\le C_1,\\
3 C_2\cdot\frac{1}{K}\sum_{k = 1}^K {\text{TV}}(P_k^{\calRO} || P^{\calD}_k) - C_1,  &  C_2\cdot\frac{1}{K}\sum_{k = 1}^K{\text{TV}}(P_k^{\calRO} || P^{\calD}_k)\ge C_1.\\
\end{array} \right.
\end{equation}
\end{thm}

\red{
}

\blue{${\text{TV}}(P_k^{\calRO} || P^{\calD}_k)$ is the total variation distance between distributions of client $k$'s original model and distorted model, and it is referred to as the \textit{distortion extent}.} There are two scenarios regarding the distance between $F^{\calRO}_k$ and $F^{\calO}_k$.

(1) If the prior belief distribution $F^{\calO}_k$ of the attacker on the private data is very close to the posterior belief distribution $F^{\calRO}_k$ of the attacker upon observing the true model parameter, $C_{1, k}$ is very small. Consider the extreme case when $F^{\calRO}_k = F^{\calO}_k$. Then $\epsilon_{p,k} = \sqrt{{\text{JS}}(F^{\calA}_k || F^{\calO}_k)} = \sqrt{{\text{JS}}(F^{\calA}_k || F^{\calRO}_k)}$. For this scenario, the larger the distortion extent is, the larger $\epsilon_{p,k} = \sqrt{{\text{JS}}(F^{\calA}_k || F^{\calRO}_k)}$ is. Therefore, the privacy leakage might increase with the extent of the distortion. This is consistent with the upper bound in the second line of \pref{eq: upper_bound_for_privacy_of_client_k}.

(2) If the prior belief distribution $F^{\calO}_k$ of the attacker on the private data is far away from the posterior belief distribution $F^{\calRO}_k$ of the attacker upon observing the true model parameter, $C_{1, k}$ is rather large. For this scenario, the larger the distortion extent is, the smaller the privacy leakage is. This is consistent with the upper bound in the first line of \pref{eq: upper_bound_for_privacy_of_client_k}.

In this work, we focus on the second scenario. From the upper bound of Bayesian privacy leakage, it is reasonable to design a robust metric for evaluating the amount of leaked information as
\begin{align}\label{eq:privacy_leakage_tv}
    \widetilde\epsilon_p = 2 C_1 - C_2\cdot\frac{1}{K}\sum_{k = 1}^K {\text{TV}}(P_k^{\calRO} || P^{\calD}_k).
\end{align}
In addition to privacy leakage, utility loss and efficiency reduction are the other two critical concerns of FL that we investigate in this work. We adopt their definitions from the work~\cite{zhang2022trading} and present them as follows.  

\paragraph{Utility Loss}
The utility loss is defined as the discrepancy between the utility obtained by using the unprotected model information $W_{\text{fed}}^{\calRO}$ following distribution $P_{\text{fed}}^{\calRO}$ and protected model information $W_{\text{fed}}^{\calD}$ following the protected distribution $P_{\text{fed}}^{\calD}$. The utility loss $\epsilon_{u,k}$ of client $k$ is defined as follows.
\begin{align*}
    \epsilon_{u,k} = \mathbb E_{W_{\text{fed}}^{\calRO}\sim P_{\text{fed}}^{\calRO}}[U_k(W^{\calRO}_{\text{fed}})] - \mathbb E_{W_{\text{fed}}^{\calD}\sim P_{\text{fed}}^{\calD}}[U_k(W_{\text{fed}}^{\calD})],
\end{align*}
where $U_k$ denotes the utility measurement of client $k$, the expectation is taken with respect to the parameters. 
Furthermore, the utility loss $\epsilon_{u}$ of the FL system is defined as:
\begin{align*}
    \epsilon_u = \mathbb E_{W_{\text{fed}}^{\calRO}\sim P_{\text{fed}}^{\calRO}}[U(W^{\calRO}_{\text{fed}})] - \mathbb E_{W_{\text{fed}}^{\calD}\sim P_{\text{fed}}^{\calD}}[U(W_{\text{fed}}^{\calD})],
\end{align*}
where $U$ denotes the utility measurement of the FL system.

Utility measurements can measure model performance for a variety of learning tasks. For example, Utility measurement measures classification accuracy in a classification problem and prediction accuracy in a regression problem.

\paragraph{Efficiency Reduction} The efficiency reduction measures the difference between the costs spent on training model drawn from the protected and the unprotected distributions $P_{k}^{\calD}$ and $P_{k}^{\calRO}$. The efficiency reduction $\epsilon_{e,k}$ of client $k$ is defined as:
\begin{align*}
    \epsilon_{e,k} = \mathbb E_{W_k^{\calD}\sim P_{k}^{\calD}}[C(W^{\calD}_k)] - \mathbb E_{W_k^{\calRO}\sim P_{k}^{\calRO}}[C(W^{\calRO}_k)],
\end{align*}
where $C$ denotes a mapping from the model information to the efficiency measured in terms of the communication cost (e.g., the transmitted bits) or the training cost. Furthermore, the efficiency reduction $\epsilon_e$ of FL system is defined as:
\begin{align*}
    \epsilon_e = \frac{1}{K}\sum_{k = 1}^K \epsilon_{e,k},
\end{align*}

In the next section, we propose a meta-learning algorithm for solving the optimization problem defined in \pref{eq: constraint_optimization_problem}.

\section{ 
A Meta-learning Algorithm for Tuning Parameterized Protection Mechanisms}\label{sec: meta_algorithm}


In general, no closed-form solution exists to solve the optimization problem defined in \pref{eq: constraint_optimization_problem}. In this section, we propose a practical meta-learning algorithm for solving this optimization problem.

We denote $\gamma$ as the protection mechanism parameter to be optimized and $\widetilde\epsilon_{p} (\gamma)$ as the upper bound of the privacy leakage. 
The optimization problem formulated in \pref{eq: constraint_optimization_problem} can be further expressed as follows:
\begin{align}\label{eq: relaxed_opt}
\begin{array}{r@{\quad}l@{}l@{\quad}l}
\quad\min\limits_{\gamma} & \eta_{u}\cdot \epsilon_{u}(\gamma) +  \eta_{e}\cdot \epsilon_{e}(\gamma)\\
\text{s.t.,} & \widetilde\epsilon_p (\gamma)\le\epsilon.\\
\end{array}
\end{align}

\blue{
The Eq.(\ref{eq: relaxed_opt}) is a relaxed form of the optimization problem defined in \pref{eq: constraint_optimization_problem} because we leverage the upper bound of the privacy leakage $\widetilde\epsilon_{p} (\gamma)$ as the way to measure the actual privacy leakage. The parameter $\gamma$ of the protection mechanism that maximizes the payoff bound is obtained by solving \pref{eq: relaxed_opt}.
In this work, we are focusing on designing a meta-learning algorithm that aims to find the optimal protection parameter for various widely adopted privacy protection mechanisms by solving the optimization problem \pref{eq: relaxed_opt}. The Assumption and Theorem for solving Eq.(\ref{eq: relaxed_opt}) are illustrated in the Assumption \ref{assump: monotonic_function} and Theorem \ref{thm: meta_alg_for_opt_problem}.}
\begin{assumption}\label{assump: monotonic_function}
    Let $\epsilon_{u} (\epsilon_{u}: \gamma \rightarrow \R)$ and $\epsilon_{e} (\epsilon_{e}: \gamma \rightarrow \R)$ represent the utility loss and efficiency reduction, which map parameter $\gamma$ to a real value. Assume that $\epsilon_{u}(\gamma)$ and $\epsilon_{e}(\gamma)$ increases with $\gamma$.
\end{assumption}

\blue{

\begin{thm}\label{thm: meta_alg_for_opt_problem} 
     Let $\epsilon_{u} (\epsilon_{u}: \gamma \rightarrow \R)$, $\epsilon_{e} (\epsilon_{e}: \gamma \rightarrow \R)$ and $\epsilon_{p} (\epsilon_{p}: \gamma \rightarrow \R)$ represent the utility loss, efficiency reduction and privacy leakage (mapping from the protection parameter $\gamma$ to a real value). Let \pref{assump: monotonic_function} hold, there is an algorithm that returns a  protection parameter, which is an optimal solution to this optimization problem formulated in \pref{eq: relaxed_opt}. 
\end{thm}


\begin{algorithm}[!h]
	\caption{Meta Algorithm for Solving Eq.(\ref{eq: relaxed_opt})}
	\begin{algorithmic}[1]
	\Statex \textbf{Input:} the protection mechanism $\mathcal{M}$, and its corresponding protection parameter $\gamma$, preferences $\eta_{u}$ and $\eta_{e}$.
    \Statex \textbf{Output:} the values of the optimal protection parameter $\gamma^*$ and objective $\ell $.

    \State Derive the upper bound of privacy leakage $\widetilde\epsilon_{p} (\gamma)$ based on the specific form of $\text{TV}(P_k^{\calRO} || P^{\calD}_k)$ of $\mathcal{M}$ according to Eq.(\ref{eq:privacy_leakage_tv}): $\widetilde\epsilon_{p} (\gamma) = 2 C_1 - C_2\cdot\frac{1}{K}\sum_{k = 1}^K {\text{TV}}(P_k^{\calRO} || P^{\calD}_k)$.
    \State Obtain the protection parameter $\gamma^*$ that satisfies $\widetilde\epsilon_p (\gamma^*) = \epsilon$.
    \State Estimate the values of $C_1$ and $C_2$ according to Algorithm \ref{alg:prepare_model} and Algorithm \ref{alg:param_estimation_1}.
    \State Calculate the value of $\gamma^*$ by bringing in the estimated values of $C_1$ and $C_2$.
    \State Calculate the value of $\ell = \eta_{u}\cdot \epsilon_{u}(\gamma^*) +  \eta_{e}\cdot \epsilon_{e}(\gamma^*)$.
	\end{algorithmic}\label{alg:meta}
\end{algorithm}

Since the privacy leakage $\widetilde\epsilon_p (\gamma)$ decreases while utility loss $\epsilon_{u}(\gamma)$ and efficiency reduction $\epsilon_{e}(\gamma)$ increases as the protection parameter $\gamma$ increases, a straightforward way to solve the optimization problem \pref{eq: relaxed_opt} is to find the protection parameter $\gamma^*$ that satisfies $\widetilde\epsilon_p (\gamma^*) = \epsilon$, at the point of which the $\eta_{u}\cdot \epsilon_{u}(\gamma^*) +  \eta_{e}\cdot \epsilon_{e}(\gamma^*)$ is minimized and the constraint $\widetilde\epsilon_p (\gamma^*)\le\epsilon$ is satisfied.

The Meta-Algorithm for solving Eq.(\ref{eq: relaxed_opt}) is described in Algorithm \ref{alg:meta}. To find the optimal protection parameter $\gamma^*$, we need first to derive the specific form of $\widetilde\epsilon_p (\gamma)$, which depends on the specific protection mechanism $\mathcal{M}$ used to protect the data privacy during federated learning (lines 1-2 of Algo. \ref{eq: relaxed_opt}). In Section \ref{sec:app_meta}, we will elaborate on how we derive the specific forms of privacy leakage and obtain the optimal protection parameters for four widely adopted privacy protection mechanisms, namely, Paillier HE, Secret Sharing, Randomization, and Compression. To obtain the values of $\gamma^*$, we need to estimate the constants $C_1$, and $C_2$, the values of which depend on the specific datasets and models clients adopt to conduct federated learning (lines 3-4 of Algo. \ref{eq: relaxed_opt}). In Section 
\ref{sec:measure}, we will provide practical methods to empirically measure $C_1$ and $C_2$.

}

\textbf{Remark:} In scenarios where preferences $\eta_{u}$ and $\eta_e$ of Eq. (\ref{eq: relaxed_opt}) are unable to be determined, one has to formulate the problem as Constrained Multi-Objective Federated Learning (CMOFL) optimization and seek the Pareto optimal solutions using CMOFL algorithms such as Non-dominated Sorting Genetic Algorithm and multi-objective Bayesian optimization. Nevertheless, these investigations are out of the scope of the present article, and we refer readers to a recent study by~\cite{kang2023cmofl}. 

\textbf{Remark:} This optimization problem of \pref{eq: constraint_optimization_problem} can be extended to satisfy the personalized requirement of the federated learning system. For example, if the utility loss of the federated learning system is required not to exceed $\varphi$, then the optimization problem is expressed as
\begin{align}\label{eq:privacy_utility_constraints}
\begin{array}{r@{\quad}l@{}l@{\quad}l}
\quad\min\limits_{\gamma} & \epsilon_{e}(\gamma)\\
\text{s.t.,} & \widetilde\epsilon_p (\gamma)\le\epsilon, \widetilde\epsilon_u (\gamma)\le\varphi.\\
\end{array}
\end{align}
where $\widetilde\epsilon_u$ represents the upper bound for utility loss, mapping from the protection parameter $\gamma$ to a real value. Similarly, the constraint might contain the upper bound for efficiency reduction. The specific forms of the upper bound for utility loss and efficiency reduction are illustrated in Appendix \ref{sec: upper_bounds_for_three_metrics}.

\section{Applications of the Meta-learning Algorithm} \label{sec:app_meta}
In this section, we derive specific measurements of privacy leakage for Randomization, Paillier Homomorphic Encryption, Secret Sharing, and Compression based on the general form of privacy leakage formulated in Eq. (\ref{eq:privacy_leakage_tv}) (Please refer to Appendix for detailed proofs of these measurements), and we elaborate on how we find the optimal protection parameter of each protection mechanism. This section is associated with lines 1-2 of Algorithm \ref{alg:meta}.

\begin{table*}[htbp]
\caption{Optimal values of protection parameters associated with protection mechanisms: Paillier HE, Secret Sharing, Randomization, and Compression. $m$ denotes the dimension of the model parameter.}
\label{tab:tradeoff-comp}
\begin{tabular}{c||c|c|c|c}
\hline
\begin{tabular}[c]{@{}c@{}}Protection \\Mechanism\end{tabular} 
& \begin{tabular}[c]{@{}c@{}}Paillier HE\end{tabular} 
&
\begin{tabular}[c]{@{}c@{}} Secret Sharing\end{tabular}
&
\begin{tabular}[c]{@{}c@{}} Randomization\end{tabular}
&
\begin{tabular}[c]{@{}c@{}} Compression\end{tabular}
\\ \hline
\begin{tabular}[c]{@{}c@{}} Parameter \end{tabular}
& \begin{tabular}[c]{@{}c@{}} $n$ \end{tabular} 
& \begin{tabular}[c]{@{}c@{}} $r$ \end{tabular} 
& \begin{tabular}[c]{@{}c@{}} $\sigma_{\epsilon}^2$ \end{tabular}
& \begin{tabular}[c]{@{}c@{}} $\rho$ \end{tabular}
\\ \hline
\begin{tabular}[c]{@{}c@{}} Meaning \end{tabular}
& \begin{tabular}[c]{@{}c@{}} length of \\the ciphertext \end{tabular} 
& \begin{tabular}[c]{@{}c@{}} bound of the \\added random \\value \end{tabular} 
& \begin{tabular}[c]{@{}c@{}} variance of the \\added random \\noise \end{tabular}
& \begin{tabular}[c]{@{}c@{}} probability of keeping \\model parameters \end{tabular}
\\ \hline
\begin{tabular}[c]{@{}c@{}} Optimal \\Value \end{tabular}
& \begin{tabular}[c]{@{}c@{}}
$\frac{(2\delta)^{\frac{1}{2}}}{[(2C_1-\epsilon)/C_2]^{\frac{1}{2m}}}$
\end{tabular} 
& \begin{tabular}[c]{@{}c@{}} $\frac{\delta}{[(2C_1 - \epsilon)/C_2]^{\frac{1}{m}}}$ \end{tabular} 
& \begin{tabular}[c]{@{}c@{}} $\frac{100\cdot (2C_1 -\epsilon)}{C_2\sqrt{\sum_{i=1}^{m}\frac{1}{\sigma_i^4}}}$\end{tabular}
& \begin{tabular}[c]{@{}c@{}} $\left(1 - \frac{2C_1 - \epsilon}{C_2}\right)^{\frac{1}{m}}$ \end{tabular}
\\ \hline
\end{tabular}
\end{table*}

\blue{We denote $\delta>0$ as a small constant in the following analysis. We formalize the protection parameters of Paillier Homomorphic Encryption, Secret Sharing, Randomization, and Compression as the length of the ciphertext $n$, the bound of the added random value $r$,  the variance of the added random noise $\sigma_{\epsilon}^2$, and the probability of keeping the parameters $\rho$, respectively. Table \ref{tab:tradeoff-comp} summarizes the optimal protection parameters for Paillier Homomorphic Encryption, Secret Sharing, Randomization, and Compression. In this section, we elaborate on how we obtain these optimal protection parameters. To facilitate the analysis, we focus on the scenario where $C_2\cdot{\text{TV}}(P_k^{\calRO} || P^{\calD}_k)\le C_{1, k}$  (see Ep.(\ref{eq:two_scenarios}) of Theorem 4.1).}

\subsection{Randomization Mechanism}



For the randomization mechanism, \blue{we have the total variation distance between distributions of client $k$'s original model and distorted model lower bounded as follows (please refer to Appendix C of \cite{zhang2022no}).}
   \begin{align}
       {\text{TV}}(P_k^{\calRO} || P^{\calD}_k)\ge \frac{1}{100}\min\left\{1, \sigma_{\epsilon}^2\sqrt{\sum_{i=1}^{m}\frac{1}{\sigma_i^4}} \right\},
   \end{align}
\blue{where $\sigma_i, i\in \{1,...,m\}$ is the variance of client $k$'s model parameters with size $m$, and $\sigma_{\epsilon}^2$ is the variance of the random noise that we want to optimize.}


\blue{With this bound on total variation distance and \pref{thm: privacy_leakage_upper_bound_main_result}, we can derive the upper bound of privacy leakage (please also refer to Eq.(\ref{eq:privacy_leakage_tv})),
\begin{align}\label{eq: privacy_leakage_expression}
    \widetilde\epsilon_p = 2 C_1 - C_2\cdot\frac{1}{K}\sum_{k = 1}^K \frac{1}{100}\min\left\{1, \sigma_{\epsilon}^2\sqrt{\sum_{i=1}^{m}\frac{1}{\sigma_i^4}} \right\},
\end{align}
where $C_1 =\frac{1}{K}\sum_{k=1}^K \sqrt{{\text{JS}}(F^{\calRO}_k || F^{\calO}_k)}$, and $C_2 = \frac{1}{2}(e^{2\xi}-1)$. 
}

\textbf{Remark: $C_1$ and $C_2$ can be evaluated empirically. \blue{We will elaborate on measuring $C_1$ and $C_2$ in Section \ref{sec:measure}.}}

\begin{thm}[The Optimal Parameter for Randomization Mechanism]\label{thm: optimal_parameter_randomization}
For the randomization mechanism, the privacy leakage is upper bounded by 
\begin{align}\label{eq: privacy_bound_app_01}
    \epsilon_p \le 2 C_1 - C_2\cdot\frac{1}{K}\sum_{k = 1}^K \frac{1}{100}\min\left\{1, \sigma_{\epsilon}^2\sqrt{\sum_{i=1}^{m}\frac{1}{\sigma_i^4}} \right\}.
\end{align}
The optimal value for the optimization problem \pref{eq: relaxed_opt} is achieved when the variance of random noise is
\begin{align*}
    \sigma_\epsilon^{2*} = \frac{100\cdot (2C_1 -\epsilon)}{C_2\sqrt{\sum_{i=1}^{m}\frac{1}{\sigma_i^4}}}.
\end{align*}
\end{thm}

We refer readers to Appendix \ref{sec:proof-random-app} for the proof of Theorem \ref{thm: optimal_parameter_randomization}.





\subsection{Paillier Homomorphic Encryption}
The Paillier homomorphic encryption (PHE)~\cite{paillier1999public} is an asymmetric additive homomorphic encryption mechanism that is widely used in FL \cite{zhang2019pefl, aono2017privacy, truex2019hybrid, cheng2021secureboost} to protect data privacy. PHE contains three parts~\cite{fang2021privacy}: key generation, encryption, and decryption. We denote ($n,g$) as the public key and ($\lambda, \mu$) as the private key. 



The \textit{encryption} process randomly selects $r$ and encrypts plaintext $m$ as:
\begin{align*}
    c = g^m\cdot r^n \text{ mod } n^2, \text{  where  } n = p \cdot q,
\end{align*}
where $c$ is the ciphertext of $m$, $p$ and $q$ are two selected primes, $g\in \mathbb Z_{n^2}^*$ is a randomly selected integer. Therefore, $n$ can divide the order of $g$.\\ 

The \textit{decryption} process uses the private key $(\lambda,\mu)$ to decrypt the ciphertext $c$ as:
\begin{align*}
    m = L(c^{\lambda}\text{ mod }n^2)\cdot \mu \text{ mod } n,
\end{align*}
where $L(x) = \frac{x-1}{n}$, $\mu = (L(g^{\lambda}\text{ mod }n^2))^{-1}\text{ mod }n$, and $\lambda = \text{lcm} (p-1, q-1)$.

The following theorem bounds privacy leakage for the PHE using the length of the ciphertext $n$. The bounds for privacy leakage decrease with the increase of $n$, which could guide the selection of $n$ to adapt to the requirements of security and efficiency.

\begin{thm}\label{thm: optimal_he}
For the Paillier mechanism, the privacy leakage is bounded by 
\begin{align}
    \epsilon_p \le 2C_1 - C_2\cdot\left[ 1 - \left(\frac{2\delta}{n^2}\right)^m\right].
\end{align}
The optimal value for the optimization problem \pref{eq: relaxed_opt} is achieved when the length of the ciphertext is:
\begin{align*}
    n^* = \frac{(2\delta)^{\frac{1}{2}}}{[(2C_1-\epsilon)/C_2]^{\frac{1}{2m}}}.
\end{align*}
\end{thm}

We refer readers to Appendix \ref{sec:proof-he} for the proof of Theorem \ref{thm: optimal_he}.


\subsection{Secret Sharing Mechanism}
Secret sharing \cite{SecShare-Adi79,SecShare-Blakley79,bonawitz2017practical} is a secure multi-party computation mechanism used to protect the data privacy of machine learning models, including linear regression, logistic regression, and recommender systems. By applying secret sharing, ${\text{TV}}(P^{\calRO}_{\text{fed}} || P^{\calD}_{\text{fed}} ) = 0$, and according to Lemma C.3 of \cite{zhang2022no}, the utility loss is equal to $0$. We denote $r$ as the upper bound of the random values added to model parameters.

\blue{

\begin{thm}\label{thm: optimal_ss}
For the secret sharing mechanism, the privacy leakage is upper bounded by 
\begin{align}\label{eq: privacy_bound_app_04}
    \epsilon_p \le 2C_1 - C_2\cdot\left[1 - \left(\frac{2\delta}{b + r}\right)^{m}\right].
\end{align}
Assuming $b = r$, the optimal value for the optimization problem \pref{eq: relaxed_opt} is achieved when the secret sharing parameter is
\begin{align*}
    r^* = \frac{\delta}{[(2C_1 - \epsilon)/C_2]^{\frac{1}{m}}}.
\end{align*}
\end{thm}

}

We refer readers to Appendix \ref{sec:proof-ss} for the proof of Theorem \ref{thm: optimal_ss}.

\subsection{Compression Mechanism}

By applying the compression mechanism~\cite{zhang2022trading}, each participating client share only a portion of its model parameters with the server for aggregation. We denote $W_k^{\calRO}(i)$ and $W_k^{\calD}(i)$ as the $i$th element of $W_k^{\calRO}$ and $W_k^{\calD}$, respectively. We assume the random variable $b_i$ follows the Bernoulli distribution. $b_i$ is equal to $1$ with probability $\rho_i$ and $0$ with probability $1 - \rho_i$. Then, each element of the protected parameter $W_k^{\calD}(i)$ is computed by
\begin{equation}
W_k^{\calD}(i) =\left\{
\begin{array}{cl}
 W_k^{\calRO}(i)/\rho_i &  \text{if } b_i = 1,\\
0,  &  \text{if } b_i = 0.\\
\end{array} \right.
\end{equation}


\blue{

\begin{thm}\label{thm: opt_parameter_compression}
For the compression mechanism, the privacy leakage is upper bounded by 
\begin{align}\label{eq: privacy_bound_app}
    \epsilon_p \le 2C_1 - C_2\cdot\left(1 - \rho^m\right).
\end{align}
The optimal value for the optimization problem \pref{eq: relaxed_opt} is achieved when the compression probability is:
\begin{align*}
    \rho^* = \left(1 - \frac{2C_1 - \epsilon}{C_2}\right)^{\frac{1}{m}}.
\end{align*}
\end{thm}
We refer readers to Appendix \ref{sec:proof-compress} for the proof of Theorem \ref{thm: opt_parameter_compression}.

}

\blue{
As reported in Table \ref{tab:tradeoff-comp}, the optimal values of the protection parameters associated with the four protection mechanisms depend on $C_1$ and $C_2$, the values of which depend on the specific datasets and machine learning models adopted by clients to conduct the federated learning. Therefore, $C_1$ and $C_2$ should be empirically measured. In Section 
\ref{sec:measure}, we provide practical methods to empirically measure $C_1$ and $C_2$.
}

\section{Practical Estimation for Privacy Leakage} \label{sec:measure}
In this section, we elaborate on the methods we propose to estimate $C_1$ and $C_2$, which is associated with lines 3-4 of Algorithm \ref{alg:meta}. First, we introduce the following assumption.
\begin{assumption}\label{assump: privacy_over_parameter}
We assume that the amount of Bayesian privacy leaked from the true parameter is maximum over all possible parameters. That is, for any $w\in\calW_k$, we have that
\begin{align}
    \frac{f_{D_k|W_k}(d|w^*)}{f_{D_k}(d)}\ge \frac{f_{D_k|W_k}(d|w)}{f_{D_k}(d)},
\end{align}
where $w^*$ represents the true model parameter. 
\end{assumption}

\subsection{Estimation for $C_{1, k} = \sqrt{{\text{JS}}(F^{\calRO}_k || F^{\calO}_k)}$}

Recall that $C_1 =\frac{1}{K}\sum_{k=1}^K \sqrt{{\text{JS}}(F^{\calRO}_k || F^{\calO}_k)}$ measures the averaged square root of JS divergence between adversary's belief distribution about the private information of client $k$ before and after observing the unprotected parameter. This constant is independent of the protection mechanisms. From the definition of JS divergence, we have that:
\begin{align*}
    {\text{JS}}(F^{\calRO}_k || F^{\calO}_k) = \frac{1}{2} \left( \int_{\mathcal{D}_k}f^{\calRO}_{D_k}(d)\log\frac{f^{\calRO}_{D_k}(d)}{\frac{1}{2}(f^{\calRO}_{D_k}(d) + f^{\calO}_{D_k}(d))}\textbf{d}\mu(d) 
    +  \int_{\mathcal{D}_k} f_{D_k}(d)\log\frac{f_{D_k}(d)}{\frac{1}{2}(f^{\calRO}_{D_k}(d) + f^{\calO}_{D_k}(d))}\textbf{d}\mu(d) \right),
\end{align*}
where $f^{\calRO}_{D_k}(d) = \int_{\mathcal{W}_k^{\calRO}} f_{{D_k}|{W_k}}(d|w)dP^{\calRO}_{k}(w) = \E_w[f_{{D_k}|{W_k}}(d|w)]$, and $f^{\calO}_{D_k}(d) = f_{D_k}(d)$.



To estimate $C_1$, we need to first estimate the values of $f^{\calRO}_{D_k}(d)$ and $f^{\calO}_{D_k}(d)$. Note that 
\begin{align}
    f^{\calRO}_{D_k}(d) = \int_{\mathcal{W}_k^{\calRO}} f_{{D_k}|{W_k}}(d|w)dP^{\calRO}_{k}(w) = \E_w[f_{{D_k}|{W_k}}(d|w)].
\end{align}
Intuitively, $f_{{D_k}|{W_k}}(d|w)$ represents the probability belief of the attacker about the private data $d$ after observing the model parameter $w$. We approximate $f^{\calRO}_{D_k}(d)$ by:
\begin{align}\label{eq: estimate_for_f_o}
    \hat f^{\calRO}_{D_k}(d) = \frac{1}{M}\sum_{m = 1}^{M}[\hat f_{{D_k}|{W_k}}(d|w_m)],
\end{align}
where $w_m$ represents the model parameter observed by the attacking mechanism at the $m$-th attacking attempt to recover a data $d$ given $w_m$.



We denote $C_{1, k} = \sqrt{{\text{JS}}(F^{\calRO}_k || F^{\calO}_k)}$. Let $\kappa_1 = f^{\calRO}_{D_k}(d)$, $\kappa_2 = f^{\calO}_{D_k}(d) = f_{D_k}(d)$, and $ \hat\kappa_1 = \hat f^{\calRO}_{D_k}(d)$. Then 
\begin{align*}
    C_{1, k}^2 = \frac{1}{2}\int_{\mathcal{D}_k}\kappa_1\log\frac{\kappa_1}{\frac{1}{2}(\kappa_1 + \kappa_2)}\textbf{d}\mu(d) +  \frac{1}{2}\int_{\mathcal{D}_k}\kappa_2\log\frac{\kappa_2}{\frac{1}{2}(\kappa_1 + \kappa_2)}\textbf{d}\mu(d).
\end{align*}

With the estimation for $\kappa_1$ and $\kappa_2$, we have:
\begin{align*}
    \hat C_{1, k}^2 = \frac{1}{2}\int_{\mathcal{D}_k}\hat\kappa_1\log\frac{\hat\kappa_1}{\frac{1}{2}(\hat\kappa_1 + \kappa_2)}\textbf{d}\mu(d) +  \frac{1}{2}\int_{\mathcal{D}_k}\kappa_2\log\frac{\kappa_2}{\frac{1}{2}(\hat\kappa_1 + \kappa_2)}\textbf{d}\mu(d).
\end{align*}

\subsection{Estimation for $C_2 = \frac{1}{2}(e^{2\xi}-1)$}

Recall that $\xi = \max_{k\in [K]} \xi_k$, where $\xi_k = \max_{w\in \mathcal{W}_k, d \in \mathcal{D}_k} \left|\log\left(\frac{f_{D_k|W_k}(d|w)}{f_{D_k}(d)}\right)\right|$ represents the maximum privacy leakage over all possible information $w$ released by client $k$. When the attacking extent is fixed, $\xi$ is a constant. With \pref{assump: privacy_over_parameter}, we know that
\begin{align}
    \xi_k & = \max_{d \in \mathcal{D}_k} \left|\log\left(\frac{f_{D_k|W_k}(d|w^*)}{f_{D_k}(d)}\right)\right|.
\end{align}


We approximate $\xi_k$ by:
\begin{align}
    \hat\xi_k & = \max_{d \in \mathcal{D}_k} \left|\log\left(\frac{\hat f_{D_k|W_k}(d|w^*)}{f_{D_k}(d)}\right)\right|.
\end{align}

Let $\hat\kappa_3=\max_{k\in [K]}\hat\xi_k$, we compute $C_2$ by:
\begin{align}
    \hat C_2 = \frac{1}{2}(e^{2\hat\kappa_3}-1).
\end{align}




The estimation of $C_1$ and $C_2$ boils down to estimating the value of $\hat f_{{D_k}|{W_k}}(d|w_m)$, which we discuss in the next section.

\subsection{Estimation for $\hat f_{{D_k}|{W_k}}(d|w_m)$}


We first generate a set of models $\{w_m\}_{m=1}^M$ in a stochastic way using Algorithm \ref{alg:prepare_model}: for generating the $m$-th model $w_m$, we run the SGD on a randomly initialized model $w_m^0$ for $R$ iterations that each iteration is using a mini-batch $\mathcal{S}$ randomly sampled from $\mathcal{D}_k = \{d_1, \cdots, d_{|\mathcal{D}_k|}\}$.

\begin{algorithm}[!htp]
    \caption{Prepare Model}
        \label{alg:prepare_model}
    \begin{algorithmic}[1]

    
    \State \textbf{Input:} $\mathcal{D}$ is the dataset.
    \For{$m= 1, \ldots, M$}
        \State randomly initialize $w_m^0$.
        \For{$r=0, 1, \ldots, R$}
         \State randomly sample a mini-batch of data $\mathcal{S}$ from $\mathcal{D}$.
          \State  $w^{r + 1}_m = w^{r}_m -\eta\cdot\frac{1}{|\mathcal{S}|}\sum_{x,y \in \mathcal{S}}\nabla_{w^{r}_m} \calL(w^r_m; x, y).$
        \EndFor
    \EndFor
    return $\{w^R_m\}_{m=1}^{M}$
    \end{algorithmic}
\end{algorithm}

We then estimate $\hat f_{{D_k}|{W_k}}(d|w_m)$ using Algorithm 2, given $\{w_m\}_{m=1}^M$.  
In Algorithm 2, a client $k$ updates its local model using a mini-batch of data $\mathcal{D}_k$. The attacker (the server in FL) then leverages DLG to reconstruct $\mathcal{D}_k$ based on the model gradient and updated model. We determine whether a DLG attack on an image is successful or not by using a function $\Omega$, which measures the similarity between the recovered data and the original one (e.g., $\Omega$ can be PSNR or SSIM). We consider the original data $d_i$ is recovered successfully, if $\Omega(\Tilde{d}_i, d_i) > t$, in which $t$ is the threshold for determining whether an original data $d_i$ is recovered or not through $\Tilde{d}_i$.


Empirically, $\hat f_{{D_k}|{W_k}}(d|w_m) = \frac{M_{d}^m}{S\cdot T}$, where $d\in\mathcal{D}_k$, $M_{d}^m$ represents the number of times data $d$ is recovered by DLG in round $m$. The $\hat f^{\calRO}_{D_k}(d)$ can be estimated by \pref{eq: estimate_for_f_o}.

\begin{algorithm}[h]
    \caption{Estimate $f_{{D_k}|{W_k}}(d|w)$}
    \label{alg:param_estimation_1}
  \begin{algorithmic}[1]
    \State  \textbf{Input:} $\mathcal{D}_k$ is a mini-batch of data from client $k$; \\$S= |\mathcal{D}_k|$ is the batch size;  $\{w_m\}_{m=1}^{M}$; \\$d_0$ represents a candidate data.
     
    
    \For{$m=1, \ldots, M$}
         \State $\nabla w_m = \frac{1}{|\mathcal{D}_k|}\sum_{x,y \in \mathcal{D}_k}\nabla_{w_m} \calL(w_m; x, y)$ 
         \State $w_m' = w_m - \eta\cdot\nabla w_m $ 
         \State $M_d^m\leftarrow 0$
          \For{$t=0, 1, \ldots, T$}
              \State $\Tilde{\mathcal{D}}_k = \text{DLG}(w_m', \nabla w_{m})$ 
               \State $\text{flag} = 0$
               \For{each pair ($\Tilde{x}_{d}$, $x_{d}$) in ($\Tilde{\mathcal{D}}_k$, $\mathcal{D}_k$)}
                 \State \gray{$\triangleright$ if $d$ is recovered by DLG, then $M_d^m \mathrel{+}= 1$.}
                     \If{$\Omega(\Tilde{x}_{d}, x_{d}) > t_0$}
                       \State $M_d^m \mathrel{+}= 1$
                        \State $\text{flag} \mathrel{+}= 1$
                     \EndIf
               \EndFor
                 \State $M_{d_0}^m \mathrel{+}= S - \text{flag}$
           \EndFor
        \For{each $d\in\mathcal{D}_k\cup\{d_0\}$}
          \State $\hat f_{{D_k}|{W_k}}(d|w_m) = \frac{M_d^m}{S\cdot T}$
        \EndFor
    \EndFor
     \State return $\hat f_{{D_k}|{W_k}}(d|w_m), m=1,2, \dots, M$
    \end{algorithmic}
\end{algorithm}





\subsection{Analysis for the Estimation Error}
In this section, we analyze the estimation error of Algorithm \ref{alg:param_estimation_1}.


\paragraph{Analysis for the Estimation Error}
In this section, we provide the analysis for the estimation error of Algorithm \ref{alg:param_estimation_1}.

   



\begin{thm}[The Estimation Error of $C_{1,k}$]\label{thm: error_for_C_1}
Assume that $|\kappa_1(d) - \hat\kappa_1(d)|\le\epsilon\kappa_1(d)$. We have that
 \begin{align}
     |\hat C_{1,k} - C_{1,k}|\le\sqrt{\frac{(1+\epsilon)\log\frac{1 + \epsilon}{1 - \epsilon}}{2} +  \frac{\epsilon + \max\left\{\log(1 + \epsilon), \log\frac{1}{(1 - \epsilon)}\right\}}{2}}. 
 \end{align}
\end{thm}

The analysis for the estimation error of $C_2$ depends on the application scenario. To facilitate the analysis, we assume that $C_2$ is known in advance. With the estimators and estimation bound for $C_{1,k}$, we are now ready to derive the estimation error for privacy leakage.
\begin{thm}\label{thm: estimation_error_privacy}
We have that
\begin{align}
    |\hat\epsilon_{p,k} - \epsilon_{p,k}|\le\frac{3}{2}\cdot\sqrt{\frac{(1+\epsilon)\log\frac{1 + \epsilon}{1 - \epsilon}}{2} +  \frac{\epsilon + \max\left\{\log(1 + \epsilon), \log\frac{1}{(1 - \epsilon)}\right\}}{2}}. 
\end{align}
\end{thm}

\blue{

\section{Discussion}
A trustworthy federated learning system should at least optimize the trade-off between privacy, utility, and efficiency~\cite{kang2023cmofl}. In this work, we formulate the trustworthy federated learning as a constrained optimization problem (see Eq.(\ref{eq: relaxed_opt})) that aims to find the optimal protection parameter of a given protection mechanism to achieve minimal utility loss and efficiency reduction under a privacy constraint. The optimization problem of Eq.(\ref{eq: relaxed_opt}) can be solved by trying out different protection parameter values using either a heuristic (e.g., grid search) or a systematic approach (e.g., genetic algorithm~\cite{Goldberg1989ga}), which typically involves high computational cost and unintended privacy leakage. 

To facilitate FL practitioners in finding the optimal protection parameter of a protection mechanism, we design a meta-algorithm that involves theoretical analysis and empirical estimation. More specifically, the theoretical analysis provides theoretical upper bounds of privacy leakage for widely adopted protection mechanisms, and FL practitioners can leverage these upper bounds as vehicles to measure the privacy leakage. On the other hand, empirical estimation provides methods to estimate scenario-dependent constants used in formulating privacy leakage measurements. Consequently, the theoretical analysis and empirical estimation together provide FL practitioners with practical and convenient tools to optimize Eq.(\ref{eq: relaxed_opt}).

It should be noted that Eq.(\ref{eq: relaxed_opt}) represents only one possible formulation of the trustworthy federated learning problem. Eq.(\ref{eq: relaxed_opt}) can be adapted and modified to meet various requirements of a federated learning system. For instance, if it is necessary to ensure that the utility loss in the federated learning system does not exceed a certain threshold, the optimization problem can be represented by Eq.(\ref{eq:privacy_utility_constraints}). In such case, it may be necessary to determine the upper bound of the utility loss in order to optimize Eq.(\ref{eq:privacy_utility_constraints}). This work specifically focuses on addressing Eq.(\ref{eq: relaxed_opt}).

}

\section{Conclusion}

In this work, we provide a meta-learning framework for finding a protection mechanism to optimize the trade-off between privacy leakage, utility loss, and efficiency reduction. Under this framework, we provide bounds for privacy leakage, utility loss, and efficiency reduction. Optimizing the trade-off among the three factors is regarded as an optimization problem, which aims to find an optimal protection mechanism that achieves the smallest utility loss and efficiency reduction, given the privacy budget. We propose a meta-learning algorithm to solve this optimization problem by finding the optimal protection parameter characterizing the protection mechanism. 

The maximum privacy leakage (MBP) is independent of the protection mechanism. It depends on the exposed information and the attacking methods. The average privacy leakage measures the average square root of JS divergence between the adversary's belief distribution about the client's private information before and after observing the unprotected parameter. We use the representative semi-honest attacking mechanism, deep leakage from gradients, and the amount of privacy leaked without any protection mechanism to provide a robust estimation for the quantities related to privacy leakage. We further analyze the widely used protection mechanisms under our framework, including randomization, Paillier homomorphic encryption, secret sharing, and compression. In this work, we design a meta-learning algorithm to find the optimal protection parameter at a single round.



\textbf{ACKNOWLEDGMENTS}
We would like to thank Shaofeng Jiang, Xueyang Wu for their helpful discussions. This work was partially supported by the National Key Research and Development Program of China under Grant 2018AAA0101100 and Hong Kong RGC TRS T41-603/20-R.

\bibliographystyle{ACM-Reference-Format}
\bibliography{references}

\newpage
\onecolumn
\appendix
\section{Notations}
\begin{table*}[!htp]
\footnotesize
  \centering
  \setlength{\belowcaptionskip}{15pt}
  \caption{Table of Notation}
  \label{table: notation}
    \begin{tabular}{cc}
    \toprule
    Notation & Meaning\cr
    \midrule\
    $\epsilon_{p}$ & Privacy leakage\cr
    $\epsilon_u$ & Utility loss\cr
    $\epsilon_e$ & Efficiency reduction\cr
    $D$ & Private information, including private data and statistical information\cr
    $W_{\text{fed}}$ & parameter for the federated model\cr
    $W^{\calRO}_k$ & Unprotected model information of client $k$\cr
    $W^{\calD}_k$ & Protected model information of client $k$\cr
 $P^{\calRO}_k$ & Distribution of unprotected model information of client $k$\cr
 $P^{\calD}_k$ & Distribution of protected model information of client $k$\cr
 $\mathcal W_k^{\calD}$ & Support of $P_k^{\calD}$\cr
 $\mathcal W_k^{\calRO}$ & Support of $P_k^{\calRO}$\cr
 $\calW_k$ & Union of the supports of $P_k^{\calD}$ and $P_k^{\calRO}$\cr
 $F^{\calO}_k$ & Adversary's prior belief distribution about the private information of client $k$\cr
 $F^{\calA}_k$ & Adversary's belief distribution about client $k$ after observing the protected private information\cr
 $F^{\calRO}_k$ & Adversary's belief distribution about client $k$ after observing the unprotected private information\cr
 $\text{JS}(\cdot||\cdot)$ & Jensen-Shannon divergence between two distributions\cr
 $\text{TV}(\cdot||\cdot)$ & Total variation distance between two distributions\cr
    \bottomrule
    \end{tabular}
\end{table*}

\section{The Algorithm for Estimating Private Class Information}
\begin{algorithm}[!htp]
    \caption{Parameter Estimation for Class}
    \begin{algorithmic}[1]
    \label{alg:param_estimation}
     \State \textbf{Input:} $\mathcal{S}_k$ is a mini-batch of data from client $k$; $|\mathcal{S}_k|$ is the batch size;  $\{w_m\}_{m=1}^{M}$; $\calC_k$ represents the class set of client $k$; $C(s)$ represents the class associated with data $s$;
     

    \For{$m=1, \ldots, M$}
    
        \State $\nabla w_m = \frac{1}{|\mathcal{S}_k|}\sum_{x,y \in \mathcal{S}_k}\nabla_{w_m} \calL(w_m; x, y)$ 
        \State $w_m' = w_m - \eta\cdot\nabla w_m $ 
        \State $M_c^m\leftarrow 0 \forall c\in\calC_k$
          \For{$t=0, 1, \ldots, T$}
          
               \State $\Tilde{\mathcal{D}}_k = \text{DLG}(w_m', \nabla w_{m})$ 
               \For{each pair ($\Tilde{s}$, ${s}$) in ($\Tilde{\mathcal{S}}_k$, $\mathcal{S}_k$)}
               
                \State $\triangleright$ test if $s$ is recovered by DLG.
                     \If{$\mathcal{M}(\Tilde{s}, {s}) > t_0$}
                       \State $M_{C(s)}^m += 1$
                     \EndIf
               \EndFor
           \EndFor
        \For{each $d\in\mathcal{C}_k$}
         \State $\hat f_{{D_k}|{W_k}}(d|w_m') = \frac{M_d^m}{|\mathcal{C}_k|\cdot T}$
        \EndFor
    \EndFor

    \State return $\hat f_{{D_k}|{W_k}}(d|w_m')$
    \end{algorithmic}
\end{algorithm}





\section{Analysis for Theorem \ref{thm: error_for_C_1}}
To facilitate the analysis of the error bound, we assume that the data and the model parameter are discrete. 
\begin{lem}
With probability at least $1 - 2\sum_{d\in\mathcal{D}_k}\exp\left(\frac{-\epsilon^2 T \kappa_1(d)}{3}\right)$, we have that 
\begin{align}
    |\hat\kappa_1(d) - \kappa_1(d)|\le\epsilon\kappa_1(d). 
\end{align}    
\end{lem}
\begin{proof}
Let 
\begin{align}
   \kappa_1(d) & = f_{D_k}^{\calRO}(d) \\
   & = \int_{\mathcal{W}_k^{\calRO}} f_{{D_k}|{W_k}}(d|w)dP^{\calRO}_{k}(w)\\
   & = f_{{D_k}|{W_k}}(d|w^*),
\end{align}
where the third equality is due to $P^{\calRO}_{k}$ is a degenerate distribution.


Notice that
\begin{align}
    \hat\kappa_1(d) = \hat f_{D_k}^{\calRO}(d) = \hat f_{{D_k}|{W_k}}(d|w^*) = \frac{1}{T}\sum_{t = 1}^T \one\{\text{d is recovered in $t$-th round}\}.
\end{align}

Therefore,
\begin{align}
   \E[\hat\kappa_1(d)] = \E[\hat f_{{D_k}|{W_k}}(d|w^*)] = \kappa_1(d).
\end{align}

We also assume that 
\begin{align}
   |\kappa_1(d) - \hat\kappa_1(d)|\le\epsilon\kappa_1(d).
\end{align}

Using Multiplicative Chernoff bound (\cite{motwani1996randomized}), we have that
\begin{align}
    \Pr\left[|T\hat\kappa_1(d) - \E[T\hat\kappa_1(d)]|\ge\epsilon\E[T\hat\kappa_1(d)]\right]\le 2\exp\left(\frac{-\epsilon^2\E[T\hat\kappa_1(d)]}{3}\right).
\end{align}

Therefore, with probability at least $1 - 2\sum_{d\in\mathcal{D}_k}\exp\left(\frac{-\epsilon^2 T \kappa_1(d)}{3}\right)$, we have that 
\begin{align}
    |\hat\kappa_1(d) - \kappa_1(d)| = |\hat\kappa_1(d) - \E[\hat\kappa_1(d)]|\le\epsilon\kappa_1(d). 
\end{align}
\end{proof}

\begin{thm}[The Estimation Error of $C_{1,k}$]
Assume that $|\kappa_1(d) - \hat\kappa_1(d)|\le\epsilon\kappa_1(d)$. We have that
 \begin{align}
     |\hat C_{1,k} - C_{1,k}|\le\sqrt{\frac{(1+\epsilon)\log\frac{1 + \epsilon}{1 - \epsilon}}{2} +  \frac{\epsilon + \max\left\{\log(1 + \epsilon), \log\frac{1}{(1 - \epsilon)}\right\}}{2}}. 
 \end{align}
\end{thm}

\begin{proof}

Recall that
\begin{align*}
    \hat C_{1,k}^2 = \frac{1}{2}\sum_{d\in\mathcal{D}_k}\hat\kappa_1(d)\log\frac{\hat\kappa_1(d)}{\frac{1}{2}(\hat\kappa_1(d) + \kappa_2(d))} +  \frac{1}{2}\sum_{d\in\mathcal{D}_k}\kappa_2(d)\log\frac{\kappa_2(d)}{\frac{1}{2}(\hat\kappa_1(d) + \kappa_2(d))},
\end{align*}

and

\begin{align*}
    C_{1,k}^2 = \frac{1}{2}\sum_{d\in\mathcal{D}_k}\kappa_1(d)\log\frac{\kappa_1(d)}{\frac{1}{2}(\kappa_1(d) + \kappa_2(d))}  +  \frac{1}{2}\sum_{d\in\mathcal{D}_k}\kappa_2(d)\log\frac{\kappa_2(d)}{\frac{1}{2}(\kappa_1(d) + \kappa_2(d))} .
\end{align*}

Therefore, we have
\begin{align}\label{eq: upper_bound_of_gap_of_c_1}
    &|\hat C_{1,k}^2 - C_{1,k}^2|\le \left|\frac{1}{2}\sum_{d\in\mathcal{D}_k}\hat\kappa_1(d)\log\frac{\hat\kappa_1(d)}{\frac{1}{2}(\hat\kappa_1(d) + \kappa_2(d))} - \frac{1}{2}\sum_{d\in\mathcal{D}_k}\kappa_1(d)\log\frac{\kappa_1(d)}{\frac{1}{2}(\kappa_1(d) + \kappa_2(d))}\right|\nonumber\\ 
    &+ \left|\frac{1}{2}\sum_{d\in\mathcal{D}_k}\kappa_2(d)\log\frac{\kappa_2(d)}{\frac{1}{2}(\hat\kappa_1(d) + \kappa_2(d))} - \frac{1}{2}\sum_{d\in\mathcal{D}_k}\kappa_2(d)\log\frac{\kappa_2(d)}{\frac{1}{2}(\kappa_1(d) + \kappa_2(d))}\right|
\end{align}

\paragraph{Bounding Term 1 of RHS of \pref{eq: upper_bound_of_gap_of_c_1}}

First, we consider the relationship between $\kappa_1(d)$ and $\kappa_2(d)$.
Recall that 
\begin{align}
   \kappa_1(d) & = f_{D_k}(d) \\
   & = \int_{\mathcal{W}_k^{\calRO}} f_{{D_k}|{W_k}}(d|w)dP^{(k)}(w)\\
   & = f_{{D_k}|{W_k}}(d|w^*),
\end{align}
where the third equality is due to $P^{\calRO}$ is a degenerate distribution.

Let $\kappa_2(d) = f^{\calO}_{D_k}(d) = f_{D_k}(d)$. 

Notice that
\begin{align}
    \hat\kappa_1(d) = \hat f_{D_k}(d) = \hat f_{{D_k}|{W_k}}(d|w^*) = \frac{1}{T}\sum_{t = 1}^T \one\{\text{d is recovered in $t$-th round}\}.
\end{align}

Given that $|\kappa_1(d) - \hat\kappa_1(d)|\le\epsilon\kappa_1(d)$, we have that
\begin{align}
   (1 - \epsilon)\kappa_1(d)\le\hat\kappa_1(d)\le (1 + \epsilon)\kappa_1(d).
\end{align}

We have that 
\begin{align}
    \hat\kappa_1(d)\log\frac{2\hat\kappa_1(d)}{\hat\kappa_1(d) + \kappa_2(d)}\le (1 + \epsilon)\kappa_1(d)\log\frac{2(1 + \epsilon)\kappa_1(d)}{(1 - \epsilon)\kappa_1(d) + \kappa_2(d)}
\end{align}

Therefore, we have that
\begin{align}
    &\frac{1}{2}\sum_{d\in\mathcal{D}_k}\hat\kappa_1(d)\log\frac{\hat\kappa_1(d)}{\frac{1}{2}(\hat\kappa_1(d) + \kappa_2(d))} - \frac{1}{2}\sum_{d\in\mathcal{D}_k}\kappa_1(d)\log\frac{\kappa_1(d)}{\frac{1}{2}(\kappa_1(d) + \kappa_2(d))}\\
    &\le\frac{1}{2}\sum_{d\in\mathcal{D}_k}(1 + \epsilon)\kappa_1(d)\log\frac{2(1 + \epsilon)\kappa_1(d)}{((1 - \epsilon)\kappa_1(d) + \kappa_2(d))} - \frac{1}{2}\sum_{d\in\mathcal{D}_k}\kappa_1(d)\log\frac{2\kappa_1(d)}{(\kappa_1(d) + \kappa_2(d))}\\
    & = \frac{1}{2}\sum_{d\in\mathcal{D}_k}\kappa_1(d)\log\frac{2(1 + \epsilon)\kappa_1(d)}{((1 - \epsilon)\kappa_1(d) + \kappa_2(d))} + \frac{1}{2}\sum_{d\in\mathcal{D}_k}\epsilon\kappa_1(d)\log\frac{2(1 + \epsilon)\kappa_1(d)}{((1 - \epsilon)\kappa_1(d) + \kappa_2(d))}\\ 
    & - \frac{1}{2}\sum_{d\in\mathcal{D}_k}\kappa_1(d)\log\frac{2\kappa_1(d)}{(\kappa_1(d) + \kappa_2(d))}\\
    &=\frac{1}{2}\sum_{d\in\mathcal{D}_k}\kappa_1(d)\log\frac{1 + \epsilon}{1 - \epsilon} + \epsilon\frac{1}{2}\sum_{d\in\mathcal{D}_k}\kappa_1(d)\log\frac{2(1 + \epsilon)\kappa_1(d)}{(1 - \epsilon)(\kappa_1(d) + \kappa_2(d))}\\
    &= \frac{(1+\epsilon)\log\frac{1 + \epsilon}{1 - \epsilon}}{2} + \epsilon \frac{1}{2}\sum_{d\in\mathcal{D}_k}\kappa_1(d)\log\frac{\kappa_1(d)}{\frac{1}{2}(\kappa_1(d) + \kappa_2(d))},
\end{align}
where the last equation is due to $\sum_{d\in\mathcal{D}_k}\kappa_1(d) = 1$ from the definition of $\kappa_1(d)$.

On the other hand, we have that
\begin{align}
    &\frac{1}{2}\sum_{d\in\mathcal{D}_k}\hat\kappa_1(d)\log\frac{\hat\kappa_1(d)}{\frac{1}{2}(\hat\kappa_1(d) + \kappa_2(d))} - \frac{1}{2}\sum_{d\in\mathcal{D}_k}\kappa_1(d)\log\frac{\kappa_1(d)}{\frac{1}{2}(\kappa_1(d) + \kappa_2(d))}\\
    &\ge\frac{1}{2}\sum_{d\in\mathcal{D}_k}(1 - \epsilon)\kappa_1(d)\log\frac{2(1 - \epsilon)\kappa_1(d)}{((1 + \epsilon)\kappa_1(d) + \kappa_2(d))} - \frac{1}{2}\sum_{d\in\mathcal{D}_k}\kappa_1(d)\log\frac{2\kappa_1(d)}{(\kappa_1(d) + \kappa_2(d))}\\
    & = \frac{1}{2}\sum_{d\in\mathcal{D}_k}\kappa_1(d)\log\frac{2(1 - \epsilon)\kappa_1(d)}{((1 + \epsilon)\kappa_1(d) + \kappa_2(d))} - \frac{1}{2}\sum_{d\in\mathcal{D}_k}\epsilon\kappa_1(d)\log\frac{2(1 - \epsilon)\kappa_1(d)}{((1 + \epsilon)\kappa_1(d) + \kappa_2(d))}\\ 
    & - \frac{1}{2}\sum_{d\in\mathcal{D}_k}\kappa_1(d)\log\frac{2\kappa_1(d)}{(\kappa_1(d) + \kappa_2(d))}\\
    &=\frac{1}{2}\sum_{d\in\mathcal{D}_k}\kappa_1(d)\log\frac{1 - \epsilon}{1 + \epsilon} - \epsilon\frac{1}{2}\sum_{d\in\mathcal{D}_k}\kappa_1(d)\log\frac{2(1 - \epsilon)\kappa_1(d)}{(1 + \epsilon)(\kappa_1(d) + \kappa_2(d))}\\
    &= \frac{(1-\epsilon)\log\frac{1 - \epsilon}{1 + \epsilon}}{2} - \epsilon \frac{1}{2}\sum_{d\in\mathcal{D}_k}\kappa_1(d)\log\frac{\kappa_1(d)}{\frac{1}{2}(\kappa_1(d) + \kappa_2(d))},
\end{align}
where the last equation is due to $\sum_{d\in\mathcal{D}_k}\kappa_1(d) = 1$ from the definition of $\kappa_1(d)$.

\paragraph{Bounding Term 2 of RHS of \pref{eq: upper_bound_of_gap_of_c_1}}

We have that
\begin{align}
    &\frac{1}{2}\sum_{d\in\mathcal{D}_k}\kappa_2(d)\log\frac{\kappa_2(d)}{\frac{1}{2}(\hat\kappa_1(d) + \kappa_2(d))} - \frac{1}{2}\sum_{d\in\mathcal{D}_k}\kappa_2(d)\log\frac{\kappa_2(d)}{\frac{1}{2}(\kappa_1(d) + \kappa_2(d))}\\
    &\le\frac{1}{2}\sum_{d\in\mathcal{D}_k}\kappa_2(d)\log\frac{\kappa_2(d)}{\frac{1}{2}((1 - \epsilon)\kappa_1(d) + \kappa_2(d))} - \frac{1}{2}\sum_{d\in\mathcal{D}_k}\kappa_2(d)\log\frac{\kappa_2(d)}{\frac{1}{2}(\kappa_1(d) + \kappa_2(d))}\\
    &=\frac{1}{2}\sum_{d\in\mathcal{D}_k}\kappa_2(d)\log\frac{2\kappa_2(d)}{(1 - \epsilon)(\kappa_1(d) + \kappa_2(d))} - \frac{1}{2}\sum_{d\in\mathcal{D}_k}\kappa_2(d)\log\frac{2\kappa_2(d)}{(\kappa_1(d) + \kappa_2(d))}\\
    & = \frac{1}{2}\sum_{d\in\mathcal{D}_k}\kappa_2(d)\log\frac{1}{(1 - \epsilon)}\\
    & = \frac{1}{2}\log\frac{1}{(1 - \epsilon)}.
\end{align}

On the other hand, we have that
\begin{align}
    &\frac{1}{2}\sum_{d\in\mathcal{D}_k}\kappa_2(d)\log\frac{\kappa_2(d)}{\frac{1}{2}(\hat\kappa_1(d) + \kappa_2(d))} - \frac{1}{2}\sum_{d\in\mathcal{D}_k}\kappa_2(d)\log\frac{\kappa_2(d)}{\frac{1}{2}(\kappa_1(d) + \kappa_2(d))}\\
    &\ge\frac{1}{2}\sum_{d\in\mathcal{D}_k}\kappa_2(d)\log\frac{\kappa_2(d)}{\frac{1}{2}((1 + \epsilon)\kappa_1(d) + \kappa_2(d))} - \frac{1}{2}\sum_{d\in\mathcal{D}_k}\kappa_2(d)\log\frac{\kappa_2(d)}{\frac{1}{2}(\kappa_1(d) + \kappa_2(d))}\\
    &=\frac{1}{2}\sum_{d\in\mathcal{D}_k}\kappa_2(d)\log\frac{2\kappa_2(d)}{(1 + \epsilon)(\kappa_1(d) + \kappa_2(d))} - \frac{1}{2}\sum_{d\in\mathcal{D}_k}\kappa_2(d)\log\frac{2\kappa_2(d)}{(\kappa_1(d) + \kappa_2(d))}\\
    & = \frac{1}{2}\sum_{d\in\mathcal{D}_k}\kappa_2(d)\log\frac{1}{(1 + \epsilon)}\\
    & = \frac{1}{2}\log\frac{1}{(1 + \epsilon)}.
\end{align}

Therefore, we have
\begin{align}
    |\hat C_{1,k}^2 - C_{1,k}^2|
    &\le \frac{(1+\epsilon)\log\frac{1 + \epsilon}{1 - \epsilon}}{2} + \epsilon \frac{1}{2}\sum_{d\in\mathcal{D}_k}\kappa_1(d)\log\frac{\kappa_1(d)}{\frac{1}{2}(\kappa_1(d) + \kappa_2(d))}\\ 
    & + \frac{1}{2}\max\left\{\log(1 + \epsilon), \log\frac{1}{(1 - \epsilon)}\right\}\\
    &\le\frac{(1+\epsilon)\log\frac{1 + \epsilon}{1 - \epsilon}}{2} +  \frac{\epsilon + \max\left\{\log(1 + \epsilon), \log\frac{1}{(1 - \epsilon)}\right\}}{2}.
\end{align}

Since
\begin{align}
   |\hat C_{1,k} - C_{1,k}|^2\le |\hat C_{1,k}^2 - C_{1,k}^2|. 
\end{align}

We have that
\begin{align}
   |\hat C_{1,k} - C_{1,k}|
   &\le\sqrt{\frac{(1+\epsilon)\log\frac{1 + \epsilon}{1 - \epsilon}}{2} +  \frac{\epsilon + \max\left\{\log(1 + \epsilon), \log\frac{1}{(1 - \epsilon)}\right\}}{2}}.
\end{align}

\end{proof}

\section{Analysis for Theorem \ref{thm: estimation_error_privacy}}
\begin{thm}\label{thm: estimation_error_privacy_app}
We have that
\begin{align}
    |\hat\epsilon_{p,k} - \epsilon_{p,k}|\le\frac{3}{2}\cdot\sqrt{\frac{(1+\epsilon)\log\frac{1 + \epsilon}{1 - \epsilon}}{2} +  \frac{\epsilon + \max\left\{\log(1 + \epsilon), \log\frac{1}{(1 - \epsilon)}\right\}}{2}}. 
\end{align}
\end{thm}
\begin{proof}
    \begin{align}
       \epsilon_{p,k} = \frac{3}{2} C_{1,k} - C_2\cdot\frac{1}{K}\sum_{k = 1}^K {\text{TV}}(P_k^{\calRO} || P^{\calD}_k),
    \end{align}
    and
    \begin{align}
       \hat\epsilon_{p,k} = \frac{3}{2} \hat C_{1,k} - C_2\cdot\frac{1}{K}\sum_{k = 1}^K {\text{TV}}(P_k^{\calRO} || P^{\calD}_k).
    \end{align}
Therefore,
\begin{align}
    |\hat\epsilon_{p,k} - \epsilon_{p,k}| 
    &= \frac{3}{2} |C_{1,k} - \hat C_{1,k}|\\
    &\le\frac{3}{2}\cdot\sqrt{\frac{(1+\epsilon)\log\frac{1 + \epsilon}{1 - \epsilon}}{2} +  \frac{\epsilon + \max\left\{\log(1 + \epsilon), \log\frac{1}{(1 - \epsilon)}\right\}}{2}}. 
\end{align}

\end{proof}

\section{Upper Bounds for Privacy Leakage, Utility Loss and Efficiency Reduction}\label{sec: upper_bounds_for_three_metrics}

\begin{lem}\cite{zhang2022no}\label{lem: JSBound}
Let $P_k^{\calRO}$ and $P^{\calD}_k$ represent the distribution of the parameter of client $k$ before and after being protected. Let $F^{\calA}_k$ and $F^{\calRO}_k$ represent the belief of client $k$ about $D$ after observing the protected and original parameter. Then we have
\begin{align*}
{\text{JS}}(F^{\calA}_k || F^{\calRO}_k)\le \frac{1}{4}(e^{2\xi}-1)^2{\text{TV}}(P_k^{\calRO} || P^{\calD}_k)^2. 
\end{align*}
\end{lem} 


The distance between the true posterior belief and the distorted posterior belief of the attacker depicts the valid information gain of the attacker, and therefore is used to measure the privacy leakage of the protector. The following theorem provides upper bounds for privacy leakage of each client and the federated learning system separately.
\begin{thm}\label{thm: privacy_leakage_upper_bound_app}
We denote $C_{1,k} = \sqrt{{\text{JS}}(F^{\calRO}_k || F^{\calO}_k)}$, $C_1 =\frac{1}{K}\sum_{k=1}^K \sqrt{{\text{JS}}(F^{\calRO}_k || F^{\calO}_k)}$, and denote $C_2 = \frac{1}{2}(e^{2\xi}-1)$, where $\xi = \max_{k\in [K]} \xi_k$, $\xi_k = \max_{w\in \mathcal{W}_k, d \in \mathcal{D}_k} \left|\log\left(\frac{f_{D_k|W_k}(d|w)}{f_{D_k}(d)}\right)\right|$ represents the maximum privacy leakage over all possible information $w$ released by client $k$, and $[K] = \{1,2,\cdots, K\}$.
Let $P_k^{\calRO}$ and $P^{\calD}_k$ represent the distribution of the parameter of client $k$ before and after being protected. Let $F^{\calA}_k$ and $F^{\calRO}_k$ represent the belief of client $k$ about $D$ after observing the protected and original parameter. The upper bound for the privacy leakage of client $k$ is
\begin{align}
    \epsilon_{p,k}\le\max\left\{2C_{1,k} - C_2\cdot{\text{TV}}(P_k^{\calRO} || P^{\calD}_k), 3 C_2\cdot{\text{TV}}(P_k^{\calRO} || P^{\calD}_k) - C_{1,k}\right\}. 
\end{align}

Specifically, 
\begin{equation}
\epsilon_{p,k}\le\left\{
\begin{array}{cl}
2C_{1,k} - C_2\cdot{\text{TV}}(P_k^{\calRO} || P^{\calD}_k), &  C_2\cdot{\text{TV}}(P_k^{\calRO} || P^{\calD}_k)\le C_{1,k},\\
3 C_2\cdot{\text{TV}}(P_k^{\calRO} || P^{\calD}_k) - C_{1,k},  &  C_2\cdot{\text{TV}}(P_k^{\calRO} || P^{\calD}_k)\ge C_{1,k}.\\
\end{array} \right.
\end{equation}

The upper bound for the privacy leakage of federated learning system is 
\begin{align*}
    \epsilon_p \le \max\left\{2C_1 - C_2\cdot\frac{1}{K}\sum_{k = 1}^K {\text{TV}}(P_k^{\calRO} || P^{\calD}_k), 3 C_2\cdot\frac{1}{K}\sum_{k = 1}^K {\text{TV}}(P_k^{\calRO} || P^{\calD}_k) - C_1\right\}.
\end{align*}

Specifically, 
\begin{equation}
\epsilon_{p}\le\left\{
\begin{array}{cl}
2C_1 - C_2\cdot\frac{1}{K}\sum_{k = 1}^K {\text{TV}}(P_k^{\calRO} || P^{\calD}_k), &  C_2\cdot\frac{1}{K}\sum_{k = 1}^K{\text{TV}}(P_k^{\calRO} || P^{\calD}_k)\le C_1,\\
3 C_2\cdot\frac{1}{K}\sum_{k = 1}^K {\text{TV}}(P_k^{\calRO} || P^{\calD}_k) - C_1,  &  C_2\cdot\frac{1}{K}\sum_{k = 1}^K{\text{TV}}(P_k^{\calRO} || P^{\calD}_k)\ge C_1.\\
\end{array} \right.
\end{equation}
\end{thm}

\begin{proof}

Notice that the square root of the Jensen-Shannon divergence satisfies the triangle inequality. 
We denote $C_{1,k} = \sqrt{{\text{JS}}(F^{\calRO}_k || F^{\calO}_k)}$, and denote $C_2 = \frac{1}{2}(e^{2\xi}-1)$, where $\xi = \max_{k\in [K]} \xi_k$, $\xi_k = \max_{w\in \mathcal{W}_k, d \in \mathcal{D}_k} \left|\log\left(\frac{f_{D_k|W_k}(d|w)}{f_{D_k}(d)}\right)\right|$ represents the maximum privacy leakage over all possible information $w$ released by client $k$, and $[K] = \{1,2,\cdots, K\}$. Fixing the attacking extent, then $\xi$ is a constant. 

\paragraph{Case 1: If $C_2\cdot{\text{TV}}(P_k^{\calRO} || P^{\calD}_k) < \sqrt{{\text{JS}}(F^{\calRO}_k || F^{\calO}_k)}$}
We have that

\begin{align*}
    \epsilon_{p,k} = \sqrt{{\text{JS}}(F^{\calA}_k || F^{\calO}_k)}&\le\sqrt{{\text{JS}}(F^{\calRO}_k || F^{\calO}_k)} + \sqrt{{\text{JS}}(F^{\calA}_k || F^{\calRO}_k)}\\
    &\le\sqrt{{\text{JS}}(F^{\calRO}_k || F^{\calO}_k)} + C_2{\text{TV}}(P_k^{\calRO} || P^{\calD}_k)\\
    &\le 2\sqrt{{\text{JS}}(F^{\calRO}_k || F^{\calO}_k)} - C_2{\text{TV}}(P_k^{\calRO} || P^{\calD}_k)\\
    & = 2C_{1,k} - C_2\cdot{\text{TV}}(P_k^{\calRO} || P^{\calD}_k),
\end{align*}
where the second inequality is due to \pref{lem: JSBound}, and 
the third inequality is due to $C_2\cdot{\text{TV}}(P_k^{\calRO} || P^{\calD}_k)\le \sqrt{{\text{JS}}(F^{\calRO}_k || F^{\calO}_k)}$ in this scenario.

\paragraph{Case 2: If $C_2\cdot{\text{TV}}(P_k^{\calRO} || P^{\calD}_k)\ge \sqrt{{\text{JS}}(F^{\calRO}_k || F^{\calO}_k)}$}

We have that
\begin{align*}
    \epsilon_{p,k} = \sqrt{{\text{JS}}(F^{\calA}_k || F^{\calO}_k)}&\le\sqrt{{\text{JS}}(F^{\calRO}_k || F^{\calO}_k)} + \sqrt{{\text{JS}}(F^{\calA}_k || F^{\calRO}_k)}\\
    &\le\sqrt{{\text{JS}}(F^{\calRO}_k || F^{\calO}_k)} + C_2{\text{TV}}(P_k^{\calRO} || P^{\calD}_k)\\
    &\le 3 C_2\cdot{\text{TV}}(P_k^{\calRO} || P^{\calD}_k) - \sqrt{{\text{JS}}(F^{\calRO}_k || F^{\calO}_k)},
\end{align*}
where the third inequality is due to $C_2\cdot{\text{TV}}(P_k^{\calRO} || P^{\calD}_k)\ge \sqrt{{\text{JS}}(F^{\calRO}_k || F^{\calO}_k)}$ in this scenario.
Therefore, the upper bound for the privacy leakage of client $k$ is
\begin{align}
    \epsilon_{p,k}\le\max\left\{2C_{1,k} - C_2\cdot{\text{TV}}(P_k^{\calRO} || P^{\calD}_k), 3 C_2\cdot{\text{TV}}(P_k^{\calRO} || P^{\calD}_k) - C_{1,k}\right\} 
\end{align}

Specifically, 
\begin{equation}
\epsilon_{p,k}\le\left\{
\begin{array}{cl}
2C_{1,k} - C_2\cdot{\text{TV}}(P_k^{\calRO} || P^{\calD}_k), &  C_2\cdot{\text{TV}}(P_k^{\calRO} || P^{\calD}_k)\le C_{1,k},\\
3 C_2\cdot{\text{TV}}(P_k^{\calRO} || P^{\calD}_k) - C_{1,k},  &  C_2\cdot{\text{TV}}(P_k^{\calRO} || P^{\calD}_k)\ge C_{1,k}.\\
\end{array} \right.
\end{equation}

Therefore, the upper bound for the privacy leakage of federated learning system is 
\begin{align*}
    \epsilon_p \le \max\left\{2C_1 - C_2\cdot\frac{1}{K}\sum_{k = 1}^K {\text{TV}}(P_k^{\calRO} || P^{\calD}_k), 3 C_2\cdot\frac{1}{K}\sum_{k = 1}^K {\text{TV}}(P_k^{\calRO} || P^{\calD}_k) - C_1\right\}.
\end{align*}

Specifically, 
\begin{equation}
\epsilon_{p}\le\left\{
\begin{array}{cl}
2C_1 - C_2\cdot\frac{1}{K}\sum_{k = 1}^K {\text{TV}}(P_k^{\calRO} || P^{\calD}_k), &  C_2\cdot\frac{1}{K}\sum_{k = 1}^K{\text{TV}}(P_k^{\calRO} || P^{\calD}_k)\le C_1,\\
3 C_2\cdot\frac{1}{K}\sum_{k = 1}^K {\text{TV}}(P_k^{\calRO} || P^{\calD}_k) - C_1,  &  C_2\cdot\frac{1}{K}\sum_{k = 1}^K{\text{TV}}(P_k^{\calRO} || P^{\calD}_k)\ge C_1.\\
\end{array} \right.
\end{equation}








\end{proof}
The following lemma provides an upper bound for utility loss.
\begin{lem}[Lemma C.3 of \cite{zhang2022no}]\label{lem: eps_u_upper_bound}
If $U(w,d)\in [0, C_4]$ for any $w \in \mathcal{W}_k$ and $d \in \calRD_k$, where $k =1,\cdots, K$, then we have
\begin{align*}
    \epsilon_{u} \le C_4\cdot {\text{TV}}(P^{\calRO}_{\text{fed}} || P^{\calD}_{\text{fed}}).
\end{align*}
\end{lem}
The following lemma provides an upper bound for efficiency reduction.
\begin{lem}\label{lem: eps_e_upper_bound_app}
If $C(w)\in [0, C_5]$ for any $w \in \mathcal{W}_k$ and $d \in \calRD_k$, where $k =1,\cdots, K$, then we have
\begin{align*}
    \epsilon_{e} \le C_5\cdot\frac{1}{K}\sum_{k = 1}^K {\text{TV}}(P^{\calRO}_{k} || P^{\calD}_{k} ).
\end{align*}
\end{lem}
\begin{proof}
Let $\mathcal U_k = \{w\in\mathcal W_k: dP_{\text{fed}}^{\calD}(w) - dP_{\text{fed}}^{\calRO}(w)\ge 0\}$, and $\mathcal V_k = \{w\in\mathcal W_k: dP_{\text{fed}}^{\calD}(w) - dP_{\text{fed}}^{\calRO}(w)< 0\}$. 
Then we have

\begin{align*}
&\epsilon_{u} = \frac{1}{K}\sum_{k=1}^K \epsilon_{u,k}\\
    & = \frac{1}{K}\sum_{k=1}^K [U_k(P^{\calRO}_{\text{fed}}) - U_k(P^{\calD}_{\text{fed}})]\\  
    &= \frac{1}{K}\sum_{k=1}^K \left[\mathbb E_{w\sim P^{\calRO}_{\text{fed}}}[U_k(w)] - \mathbb E_{w\sim P^{\calD}_{\text{fed}}}[U_k(w)]\right]\\
    & = \frac{1}{K}\sum_{k=1}^K\left[\int_{\mathcal W_k} U_k(w)dP^{\calRO}_{\text{fed}}(w) - \int_{\mathcal W_k} U_k(w) dP^{\calD}_{\text{fed}}(w)\right]\\
    & = \frac{1}{K}\sum_{k=1}^K\left[\int_{\mathcal{V}_k} U_k(w)[d P^{\calRO}_{\text{fed}}(w) - d P^{\calD}_{\text{fed}}(w)] - \int_{\mathcal{U}_k} U_k(w)[d P^{\calD}_{\text{fed}}(w) - d P^{\calRO}_{\text{fed}}(w)]\right]\\
    &\le \frac{C_4}{K}\sum_{k=1}^K\int_{\mathcal{V}_k} [d P^{\calRO}_{\text{fed}}(w) - d P^{\calD}_{\text{fed}}(w)]\\
    &=C_4\cdot {\text{TV}}(P^{\calRO}_{\text{fed}} || P^{\calD}_{\text{fed}} ).
\end{align*}

\end{proof}

\section{Analysis for Randomization Mechanism} \label{sec:proof-random-app}

The commonly random noise added to model gradients for \textit{Randomization} mechanism is Gaussian noise \cite{abadi2016deep,geyer2017differentially,truex2020ldp}.

\begin{itemize}
    \item Let $W_k^{\calRO}$ be the parameter sampled from distribution $P_k^{\calRO} = \calN(\mu_0,\Sigma_0)$, where $\mu_0 \in \mathbb{R}^m$, $\Sigma_0 = \text{diag}(\sigma_{1}^2,\cdots, \sigma_{m}^2)$ is a diagonal matrix.
    \item The distorted parameter $W_k^{\calD} = W_k^{\calRO} + \epsilon_k$, where $\epsilon_k \sim \calN(0, \Sigma_\epsilon)$ and $\Sigma_\epsilon = \text{diag}(\sigma_\epsilon^2, \cdots, \sigma_\epsilon^2)$. Therefore, $W_k^{\calD}$ follows the distribution $P_k^{\calD} = \calN(\mu_0, \Sigma_0+ \Sigma_\epsilon)$.
    \item The protected parameter $W_{\text{fed}}^{\calD} = \frac{1}{K}\sum_{k=1}^K (W_k^{\calRO} + \epsilon_k)$ follows distribution $P^{\calD}_{\text{fed}} = \calN(\mu_0, \Sigma_0/K+ \Sigma_\epsilon/K)$.
\end{itemize}


The following lemmas establish upper bounds for utility loss and efficiency reduction using the variance of the noise $\sigma_\epsilon^2$.

\begin{lem}\label{lem: epsilon_u_randomization_mechanism}
For the randomization mechanism, the utility loss is upper bounded by
\begin{align*}
    \epsilon_u\le\frac{3 C_4}{2}\cdot\min\left\{1, \sigma_\epsilon^2\sqrt{\sum_{i=1}^{m}\frac{1}{\sigma_i^4}} \right\}. 
\end{align*}
\end{lem}
\begin{proof}
From Lemma C.2 of \cite{zhang2022no}, we have that
\begin{align*}
    {\text{TV}}(P^{\calRO}_{\text{fed}} || P^{\calD}_{\text{fed}} )\le \frac{3}{2}\min\left\{1, \sigma_\epsilon^2\sqrt{\sum_{i=1}^{m}\frac{1}{\sigma_i^4}} \right\}.
\end{align*}
From \pref{lem: eps_u_upper_bound}, we have that
\begin{align*}
    \epsilon_{u} \le C_4\cdot {\text{TV}}(P^{\calRO}_{\text{fed}} || P^{\calD}_{\text{fed}} ).
\end{align*}
Therefore, we have
\begin{align*}
    \epsilon_u\le\frac{3 C_4}{2}\cdot\min\left\{1, \sigma_\epsilon^2\sqrt{\sum_{i=1}^{m}\frac{1}{\sigma_i^4}} \right\}. 
\end{align*}

\end{proof}

\begin{lem}
For Randomization mechanism, the efficiency reduction is  
\begin{align*}
    \epsilon_e\le\frac{3 C_5}{2}\min\left\{1, \sigma_\epsilon^2\sqrt{\sum_{i=1}^{m}\frac{1}{\sigma_i^4}} \right\}. 
\end{align*}
\end{lem}
\begin{proof}
From Lemma C.2 of \cite{zhang2022no}, we have that

\begin{align*}
    {\text{TV}}(P^{\calRO}_{\text{fed}} || P^{\calD}_{\text{fed}} )\le\frac{3}{2}\min\left\{1, \sigma_\epsilon^2\sqrt{\sum_{i=1}^{m}\frac{1}{\sigma_i^4}} \right\}
\end{align*}

From \pref{lem: eps_e_upper_bound_app},
\begin{align*}
    \epsilon_{e} &\le C_5\cdot\frac{1}{K}\sum_{k = 1}^K {\text{TV}}(P^{\calRO}_{k} || P^{\calD}_{k} )\\
    &\le\frac{3 C_5}{2}\min\left\{1, \sigma_\epsilon^2\sqrt{\sum_{i=1}^{m}\frac{1}{\sigma_i^4}} \right\}.
\end{align*}
\end{proof}

\subsection{Optimal Variance Parameter for Randomization Mechanism}

\begin{figure}[h!]
\begin{framed}
\textbf{Optimal Variance}
\begin{align*}
\begin{array}{r@{\quad}l@{}l@{\quad}l}
\quad\min\limits_{\sigma_\epsilon^2} & \eta_{u}\cdot \widetilde \epsilon_{u}(\sigma_\epsilon^2) +  \eta_{e}\cdot \widetilde\epsilon_{e}(\sigma_\epsilon^2)\\
\text{s.t.,} & \widetilde\epsilon_p (\sigma_\epsilon^2)\le\epsilon.\\
\end{array}
\end{align*}

\end{framed}
\caption{Secure FL with Randomization Mechanism}
\label{fig: Optimal Variance_Randomization Mechanism_app_01}
\end{figure}

To guarantee the generality of the optimization framework on various measurements for utility and efficiency, we use the upper bounds of utility loss and efficiency reduction in the optimization framework. The following lemma provides an upper bound for privacy leakage, which could be further used to solve the \textit{relaxed optimization problem}. 
\begin{lem}\label{lem: privacy_leakge_randomization_upper_bound}
For Randomization mechanism, the privacy leakage is upper bounded by
\begin{align}\label{eq: privacy_bound_app_01}
    \epsilon_p &\le 2C_1 - C_2\cdot\frac{1}{100}\min\left\{1, \sigma_\epsilon^2\sqrt{\sum_{i=1}^{m}\frac{1}{\sigma_i^4}} \right\}.
\end{align}
\end{lem}
\begin{proof}
Let $W_k^{\calRO}$ be the parameter sampled from distribution $P_k^{\calRO} = \calN(\mu_0,\Sigma_0)$, where $\mu_0 \in \mathbb{R}^m$, and $\Sigma_0 = \text{diag}(\sigma_{1}^2,\cdots, \sigma_{m}^2)$ be a diagonal matrix. Let $W_k^{\calD} = W_k^{\calRO} + \epsilon_k$, where $\epsilon_k \sim \calN(0, \Sigma_\epsilon)$ and $\Sigma_\epsilon = \text{diag}(\sigma_\epsilon^2, \cdots, \sigma_\epsilon^2)$. Therefore, $W_k^{\calD}$ follows the distribution $P_k^{\calD} = \calN(\mu_0, \Sigma_0+ \Sigma_\epsilon)$. From Lemma C.1 of \cite{zhang2022no}, we have that
\begin{align*}
    {\text{TV}}(P_k^{\calRO} || P^{\calD}_k )\le\frac{3}{2}\min\left\{1, \sigma_\epsilon^2\sqrt{\sum_{i=1}^{m}\frac{1}{\sigma_i^4}} \right\}.
\end{align*}

From \pref{thm: privacy_leakage_upper_bound_app} we have that

\begin{align*}
    \epsilon_p &\le 2C_1 - C_2\cdot\frac{1}{K}\sum_{k = 1}^K {\text{TV}}(P_k^{\calRO} || P^{\calD}_k)\\
    &\le 2C_1 - C_2\cdot\frac{1}{100}\min\left\{1, \sigma_\epsilon^2\sqrt{\sum_{i=1}^{m}\frac{1}{\sigma_i^4}} \right\}.
\end{align*}

\end{proof}

\begin{lem}
Assume that $100\cdot (2C_1 -\epsilon)\le C_2$ and $\epsilon\le 2C_1$, The optimal value for the optimization problem (\pref{eq: relaxed_opt}) of the protector when the variance is
\begin{align*}
    \sigma_\epsilon^{2*} = \frac{100\cdot (2C_1 -\epsilon)}{C_2\sqrt{\sum_{i=1}^{m}\frac{1}{\sigma_i^4}}}.
\end{align*}
\end{lem}
\begin{proof}
Note that
\begin{align*}
    \widetilde U & = \eta_{u}\cdot \widetilde \epsilon_{u}(\sigma_\epsilon^2) +  \eta_{e}\cdot \widetilde\epsilon_{e}(\sigma_\epsilon^2)\\
    &\le \eta_u\cdot\frac{3 C_4}{2}\cdot\min\left\{1, \sigma_\epsilon^2\sqrt{\sum_{i=1}^{m}\frac{1}{\sigma_i^4}} \right\} + \eta_e\cdot\frac{3 C_5}{2}\min\left\{1, \sigma_\epsilon^2\sqrt{\sum_{i=1}^{m}\frac{1}{\sigma_i^4}} \right\}.
\end{align*}

The second-order derivative
\begin{align*}
  \frac{\partial ^{2} \widetilde U}{\partial (\sigma_\epsilon^{2})^2}  = 0,
\end{align*}

and the first-order derivative

\begin{align*}
    \frac{\partial \widetilde U}{\partial \sigma_\epsilon^{2}}  > 0.
\end{align*}

From Lemma C.1 of \cite{zhang2022no}, we have that
\begin{align*}
    {\text{TV}}(P_k^{\calRO} || P^{\calD}_k )\ge\frac{1}{100}\min\left\{1, \sigma_\epsilon^2\sqrt{\sum_{i=1}^{m}\frac{1}{\sigma_i^4}} \right\}.
\end{align*}

The constraint of the optimization problem is

\begin{align*}
   \widetilde\epsilon_p (\sigma_\epsilon^2) 
   & = 2C_1 - C_2\cdot\frac{1}{K}\sum_{k = 1}^K {\text{TV}}(P_k^{\calRO} || P^{\calD}_k)\\
   & \le 2C_1 - C_2\cdot\frac{1}{100}\min\left\{1, \sigma_\epsilon^2\sqrt{\sum_{i=1}^{m}\frac{1}{\sigma_i^4}} \right\}\le\epsilon.
\end{align*}

The optimal value for the optimization problem (\pref{eq: relaxed_opt}) of the protector when the variance of the noise is
\begin{align*}
    \sigma_\epsilon^{2*} = \frac{100\cdot (2C_1 -\epsilon)}{C_2\sqrt{\sum_{i=1}^{m}\frac{1}{\sigma_i^4}}}.
\end{align*}
\end{proof}

\section{Analysis for Paillier Homomorphic Encryption}\label{sec:proof-he}



The \textbf{Paillier} encryption mechanism was proposed by \cite{paillier1999public} is an asymmetric additive homomorphic encryption mechanism, which was widely applied in FL \cite{zhang2019pefl, aono2017privacy, truex2019hybrid, cheng2021secureboost}. We first follow the basic definition of Paillier algorithm in federated learning \cite{fang2021privacy, zhang2022trading}. Paillier encyption contains three parts including key generation, encryption and decryption. Let $m$ represent the plaintext, and $c$ represent the ciphertext.

\paragraph{Key Generation} Let ($n,g$) represent the public key, and ($\lambda, \mu$) represent the private key.

Select two primes $p$ and $q$ that are rather large, satisfying that $\text{gcd}(pq, (p-q)(q-1)) = 1$. Select $g$ randomly satisfying that $g\in \mathbb Z_{n^2}^*$. Let $n = p \cdot q$, $\lambda = \text{lcm} (p-1, q-1)$, and $\mu = (L(g^{\lambda}\text{ mod }n^2))^{-1}\text{ mod }n$.

\paragraph{Encryption} Randomly select $r$, and encode $m$ as
\begin{align*}
    c = g^m\cdot r^n \text{ mod } n^2,
\end{align*}
where $n = p \cdot q$, $p$ and $q$ are two selected primes. Note that $g$ is an integer selected randomly, and $g\in \mathbb Z_{n^2}^*$. Therefore, $n$ can divide the order of $g$.\\ 

\paragraph{Decryption}
\begin{align*}
    m = L(c^{\lambda}\text{ mod }n^2)\cdot \mu \text{ mod } n,
\end{align*}
where $L(x) = \frac{x-1}{n}$, $\mu = (L(g^{\lambda}\text{ mod }n^2))^{-1}\text{ mod }n$, and $\lambda = \text{lcm} (p-1, q-1)$.

Let $m$ represent the dimension of the parameter.
\begin{itemize}
    \item Let $W_k^{\calRO}$ represent the plaintext that follows a uniform distribution over $[a_k^1 - \delta, a_k^1 + \delta]\times [a_k^2 - \delta, a_k^2 + \delta]\times\cdots\times [a_k^{m} - \delta, a_k^{m} + \delta]$.
    \item Assume that the ciphertext $W_k^{\calD}$ follows a uniform distribution over $[0,n^2-1]^{m}$.
    \item Let $W_{\text{fed}}^{\calRO}$ represent the federated plaintext that follows a uniform distribution over $[\bar a^1 - \delta, \bar a^1 + \delta]\times [\bar a^2 - \delta, \bar a^2 + \delta]\cdots\times [\bar a^{m} - \delta, \bar a^{m} + \delta]$, where $\bar a_i = \sum_{k = 1}^K a_k^i$.
    \item The federated ciphertext $W_{\text{fed}}^{\calD}$ follows a uniform distribution over $[0,n^2-1]^{m}$.
\end{itemize}


\subsection{Upper Bounds of Privacy and Efficiency for Paillier Mechanism}
The following lemma provides an upper bound for efficiency reduction for Paillier mechanism.
\begin{lem}\label{lem: epsilon_p_and_epsilon_e_upper_bound_app}
For Paillier mechanism, the efficiency reduction is bounded by
\begin{align*}
    \epsilon_e \le C_5\cdot\left[1 - \left(\frac{2\delta}{n^2}\right)^{m}\right]. 
\end{align*}
\end{lem}

\begin{proof}
Let $W^{\calRO}$ represent the plaintext $m$, and $W^{\calD}$ represent the ciphertext $c$. Recall for encryption, we have that
\begin{align*}
    c = g^m\cdot r^n \text{ mod } n^2.
\end{align*}

Assume that the ciphertext $W_k^{\calD}$ follows a uniform distribution over $[0,n^2-1]^{m}$, where $m$ represents the dimension of $W_k^{\calD}$, and the plaintext $W_k^{\calRO}$ follows a uniform distribution over $[a_k^1 - \delta, a_k^1 + \delta]\times [a_k^2 - \delta, a_k^2 + \delta]\cdots\times [a_k^{m} - \delta, a_k^{m} + \delta]$, where $m$ represents the dimension of $W_k^{\calRO}$, and $a_k^i\in [0,n^2-1]$, $\forall i = 1,2, \cdots, m$. Then we have that
\begin{align*}
    \text{TV}(P^{\calRO}_k || P^{\calD}_k) 
    &= \int_{[a_k^1 - \delta, a_k^1 + \delta]}\int_{[a_k^2 - \delta, a_k^2 + \delta]}\cdots\int_{[a_k^{m} - \delta, a_k^{m} + \delta]} \left(\left(\frac{1}{2\delta}\right)^{m} - \left(\frac{1}{n^2}\right)^{m}\right) dw_1 dw_2 \cdots d{w_{m}}\\
    & = \left[\left(\frac{1}{2\delta}\right)^{m} - \left(\frac{1}{n^2}\right)^{m}\right]\cdot (2\delta)^{m}.
\end{align*}

From \pref{lem: eps_e_upper_bound_app}, we have 
\begin{align*}
    \epsilon_{e} 
    &\le C_5\cdot\frac{1}{K}\sum_{k = 1}^K {\text{TV}}(P^{\calRO}_{k} || P^{\calD}_{k} )\\
    &\le C_5\cdot\left[1 - \left(\frac{2\delta}{n^2}\right)^{m}\right].
\end{align*}
\end{proof}

For Paillier mechanism, the distorted parameter given secret key becomes the original parameter. The following lemma shows that the utility loss for Paillier mechanism is $0$.
\begin{lem}
For Paillier mechanism, the utility loss $\epsilon_u = 0$.
\end{lem}
\begin{proof}
Let $P^{\calD}_{\text{fed}}$ represent the distribution of the distorted parameter which is decrypted by the client. 
    Note that ${\text{TV}}(P^{\calRO}_{\text{fed}} || P^{\calD}_{\text{fed}} ) = 0$. From Lemma C.3 of \cite{zhang2022no}, the utility loss is equal to $0$.
\end{proof}

\subsection{Optimal Length of The Ciphertext for Paillier Mechanism}

\begin{figure}[h!]
\begin{framed}
\textbf{Optimal Length of The Ciphertext}
\begin{align*}
\begin{array}{r@{\quad}l@{}l@{\quad}l}
\quad\min\limits_{n} & \eta_{u}\cdot \widetilde \epsilon_{u}(n) +  \eta_{e}\cdot \widetilde\epsilon_{e}(n)\\
\text{s.t.,} & \widetilde\epsilon_p (n)\le\epsilon.\\
\end{array}
\end{align*}

\end{framed}
\caption{Secure FL with Paillier Mechanism}
\label{fig Secure FL with Paillier Mechanism_app_01}
\end{figure}

The following lemma provides an upper bound for privacy leakage, which could be further used to solve the \textit{relaxed optimization problem}. 
\begin{lem}\label{lem: privacy_leakge_paillier_upper_bound_app_01}
For Paillier mechanism, the privacy leakage is upper bounded by 
\begin{align}\label{eq privacy_bound_app_03}
    \epsilon_p \le 2C_1 - C_2\cdot\left[ 1 - \left(\frac{2\delta}{n^2}\right)^m\right].
\end{align}
\end{lem}
\begin{proof}

Let $W^{\calRO}$ represent the plaintext $m$, and $W^{\calD}$ represent the ciphertext $c$. Recall for encryption, we have that
\begin{align*}
    c = g^m\cdot r^n \text{ mod } n^2.
\end{align*}

Assume that the ciphertext $W_k^{\calD}$ follows a uniform distribution over $[0,n^2-1]^{m}$, where $m$ represents the dimension of $W_k^{\calD}$, and the plaintext $W_k^{\calRO}$ follows a uniform distribution over $[a_k^1 - \delta, a_k^1 + \delta]\times [a_k^2 - \delta, a_k^2 + \delta]\cdots\times [a_k^{m} - \delta, a_k^{m} + \delta]$, where $m$ represents the dimension of $W_k^{\calRO}$, and $a_k^i\in [0,n^2-1]$, $\forall i = 1,2, \cdots, m$. Then we have that
\begin{align*}
    \text{TV}(P^{\calRO}_k || P^{\calD}_k) 
    &= \int_{[a_k^1 - \delta, a_k^1 + \delta]}\int_{[a_k^2 - \delta, a_k^2 + \delta]}\cdots\int_{[a_k^{m} - \delta, a_k^{m} + \delta]} \left(\left(\frac{1}{2\delta}\right)^{m} - \left(\frac{1}{n^2}\right)^{m}\right) dw_1 dw_2 \cdots d{w_{m}}\\
    & = \left[\left(\frac{1}{2\delta}\right)^{m} - \left(\frac{1}{n^2}\right)^{m}\right]\cdot (2\delta)^{m}.
\end{align*}

From \pref{thm: privacy_leakage_upper_bound_app}, we have that

\begin{align*}
    \epsilon_p &\le 2C_1 - C_2\cdot\frac{1}{K}\sum_{k = 1}^K {\text{TV}}(P_k^{\calRO} || P^{\calD}_k)\\
    &\le 2C_1 - C_2\cdot\left[ 1 - \left(\frac{2\delta}{n^2}\right)^m\right].
\end{align*}
\end{proof}

\begin{lem}
The optimal value for the optimization problem (\pref{eq: relaxed_opt}) of the protector when the length of the ciphertext is
\begin{align*}
    n^* = \frac{(2\delta)^{\frac{1}{2}}}{[(2C_1-\epsilon)/C_2]^{\frac{1}{2m}}}.
\end{align*}
\end{lem}

\begin{proof}
    Note that
\begin{align*}
    \widetilde U = &\eta_{u}\cdot \widetilde \epsilon_{u}(n) + \eta_{e}\cdot \widetilde\epsilon_{e}(n)\\
    &= \eta_e\cdot C_5\cdot\left[1 - \left(\frac{2\delta}{n^2}\right)^{m}\right],
\end{align*}

and 
\begin{align*}
    \widetilde\epsilon_p (n) = 2C_1 - C_2\cdot\left[ 1 - \left(\frac{2\delta}{n^2}\right)^m\right]\le\epsilon.
\end{align*}

The second-order derivative
\begin{align*}
  \frac{\partial ^{2} \widetilde U}{\partial {n}^{2}}  < 0,
\end{align*}

and the first-order derivative

\begin{align*}
    \frac{\partial \widetilde U}{\partial {n}}  > 0.
\end{align*}

Therefore, The optimal value for the optimization problem (\pref{eq: relaxed_opt}) of the protector when the length of the ciphertext is
\begin{align*}
    n^* = \frac{(2\delta)^{\frac{1}{2}}}{[(2C_1-\epsilon)/C_2]^{\frac{1}{2m}}}.
\end{align*}
\end{proof}
\section{Analysis for Secret Sharing Mechanism}\label{sec:proof-ss}

Lots of MPC-based protocols (particularly secret sharing) are used to build secure machine learning models, including linear regression, logistic regression and recommend systems. \cite{SecShare-Adi79,SecShare-Blakley79,bonawitz2017practical} were developed to distribute a secret among a group of participants. For facility of analysis, we consider the case when $ K = 2$ following \cite{zhang2022trading}.

\begin{lem}[\cite{zhang2022trading}]
For secret sharing mechanism, the utility loss $\epsilon_u = 0$.
\end{lem}
For the detailed analysis of this lemma, please refer to \cite{zhang2022trading}.

The following lemma provides an upper bound for privacy leakage, which could be further used to solve the \textit{relaxed optimization problem}. 
\begin{lem}\label{lem: privacy_leakge_paillier_upper_bound_app_02}
For secret sharing mechanism, the privacy leakage is upper bounded by 
\begin{align}\label{eq: privacy_bound_app_04}
    \epsilon_p \le 2C_1 - C_2\cdot\left[1 - \left(\frac{2\delta}{b + r}\right)^{m}\right].
\end{align}
\end{lem}
\begin{proof}
Assume that the ciphertext $W_k^{\calD}$ follows a uniform distribution over $[a_k^1 - b_k^1, a_k^1 + r_k^1]\times [a_k^2 - b_k^2, a_k^2 + r_k^2]\times\cdots\times [a_k^{m} - b_k^{m}, a_k^{m} + r_k^{m}]$, where $m$ represents the dimension of $W_k^{\calD}$, and the plaintext $W_k^{\calRO}$ follows a uniform distribution over $[a_k^1 - \delta, a_k^1 + \delta]\times [a_k^2 - \delta, a_k^2 + \delta]\cdots\times [a_k^{m} - \delta, a_k^{m} + \delta]$, where $m$ represents the dimension of $W_k^{\calRO}$, and $\delta< b_k^{m}, r_k^{m}$, $\forall i = 1,2, \cdots, m$. Then we have that
\begin{align*}
    &\text{TV}(P^{\calRO}_k || P^{\calD}_k)\\ 
    &= \int_{[a_k^1 - \delta, a_k^1 + \delta]}\int_{[a_k^2 - \delta, a_k^2 + \delta]}\cdots\int_{[a_k^{m} - \delta, a_k^{m} + \delta]} \left(\left(\frac{1}{2\delta}\right)^{m} - \prod_{j = 1}^{m}\left(\frac{1}{b_k^j + r_k^j}\right)\right) dw_1 dw_2 \cdots d{w_{m}}\\
    & = \left(\left(\frac{1}{2\delta}\right)^{m} - \prod_{j = 1}^{m}\left(\frac{1}{b_k^j + r_k^j}\right)\right)\cdot (2\delta)^{m}.
\end{align*}

For simplicity we assume that $b_k^1 = \cdots = b_k^m = b$, and $r_k^1 = \cdots = r_k^m = r$. From \pref{thm: privacy_leakage_upper_bound_app} we have that

\begin{align*}
    \epsilon_p &\le 2C_1 - C_2\cdot\frac{1}{K}\sum_{k = 1}^K {\text{TV}}(P_k^{\calRO} || P^{\calD}_k)\\
    &\le 2C_1 - C_2\cdot\left[1 - \left(\frac{2\delta}{b + r}\right)^{m}\right].
\end{align*}

\end{proof}



\begin{lem}
For the secret sharing mechanism, the efficiency reduction
\begin{align*}
    \epsilon_e \le K \cdot m \cdot \log(r).
\end{align*}
\end{lem}

\begin{proof}
   The total number of communication bits is $O(K\cdot m)$, where $K$ represents the number of clients, and $m$ represents the dimension of the parameter. Therefore, we have that
   \begin{align*}
   \epsilon_e  \le \sum_{k = 1}^K m \cdot \log(r) = K \cdot m \cdot \log(r).
   \end{align*}
 
\end{proof}

\subsection{Optimal Parameter for Secret Sharing}

\begin{figure}[h!]
\begin{framed}
\textbf{Optimal Upper Bound of The Mask}
\begin{align*}
\begin{array}{r@{\quad}l@{}l@{\quad}l}
\quad\min\limits_{r} & \eta_{u}\cdot \widetilde\epsilon_{u}(r) +  \widetilde\eta_{e}\cdot \epsilon_{e}(r)\\
\text{s.t.,} & \widetilde\epsilon_p (r)\le\epsilon.\\
\end{array}
\end{align*}

\end{framed}
\caption{Secure FL with Secret Sharing Mechanism}
\label{fig: Secure FL with Secret Sharing Mechanism_app_01}
\end{figure}

\begin{thm}
The optimal value for the optimization problem (\pref{eq: relaxed_opt}) of the protector when the secret sharing parameter is
\begin{align*}
    r^* = \frac{\delta}{[(2C_1 - \epsilon)/C_2]^{\frac{1}{m}}}.
\end{align*}
\end{thm}
\begin{proof}
    Note that
\begin{align*}
    U & = \eta_{u}\cdot \widetilde\epsilon_{u}(r) +  \eta_{e}\cdot \widetilde\epsilon_{e}(r)\\
    &\le \eta_e \cdot K \cdot m \cdot \log(r).
\end{align*}

\begin{align*}
   \widetilde\epsilon_p (r) = 2C_1 - C_2\cdot\left[1 - \left(\frac{2\delta}{b + r}\right)^{m}\right]\le\epsilon.
\end{align*}

The second-order derivative
\begin{align*}
  \frac{\partial ^{2} \widetilde U}{\partial {r}^{2}}  = 0,
\end{align*}

and the first-order derivative

\begin{align*}
    \frac{\partial \widetilde U}{\partial {r}}  > 0.
\end{align*}

Assuming $b = r$, the optimal value for the optimization problem (\pref{eq: relaxed_opt}) of the protector when the secret sharing parameter is
\begin{align*}
    r^* = \frac{\delta}{[(2C_1 - \epsilon)/C_2]^{\frac{1}{m}}}.
\end{align*}
\end{proof}

\section{Analysis for Compression Mechanism}\label{sec:proof-compress}

We consider the compression mechanism introduced in \cite{zhang2022trading}. Let $b_i$ be a random variable sampled from Bernoulli distribution. The probability that $b_i$ is equal to $1$ is $\rho_i$.
\begin{equation}
W_k^{\calD}(i) =\left\{
\begin{array}{cl}
 W_k^{\calRO}(i)/\rho_i &  \text{if } b_i = 1,\\
0,  &  \text{if } b_i = 0.\\
\end{array} \right.
\end{equation}

Let $m$ represent the dimension of the parameter.
\begin{itemize}
    \item Let $W_k^{\calRO}$ represent the original parameter that follows a uniform distribution over $[a_k^1 - \delta, a_k^1 + \delta]\times [a_k^2 - \delta, a_k^2 + \delta]\cdots\times [a_k^{m} - \delta, a_k^{m} + \delta]$.
    \item Assume that the ciphertext $W_k^{\calD}$ follows a uniform distribution over $[0,n^2-1]^{m}$.
    \item Let $W_{\text{fed}}^{\calRO}$ represent the federated plaintext that follows a uniform distribution over $[\widetilde a^1 - \delta, \widetilde a^1 + \delta]\times [\widetilde a^2 - \delta, \widetilde a^2 + \delta]\cdots\times [\widetilde a^{m} - \delta, \widetilde a^{m} + \delta]$, where $\widetilde a_i = \sum_{k = 1}^K a_k^i$.
    \item The federated ciphertext $W_{\text{fed}}^{\calD}$ follows a uniform distribution over $[0,n^2-1]^{m}$.
\end{itemize}


\begin{lem}
The utility loss is bounded by 
\begin{align*}
    \epsilon_u \le C_4\cdot \left(\left(\frac{1}{2\delta}\right)^{m} - \prod_{i = 1}^{m}\left(\frac{\rho_i}{2\delta}\right)\right)\cdot (2\delta)^{m}.
\end{align*}
\end{lem}
\begin{proof}
Assume that $W_k^{\calRO}$ follows a uniform distribution over $[a_k^1 - \delta, a_k^1 + \delta]\times [a_k^2 - \delta, a_k^2 + \delta]\cdots\times [a_k^{m} - \delta, a_k^{m} + \delta]$, where $m$ represents the dimension of $W_k^{\calRO}$. Besides,

\begin{equation}
W_k^{\calD}(i) =\left\{
\begin{array}{cl}
 W_k^{\calRO}(i)/\rho_i &  \text{if } b_i = 1,\\
0,  &  \text{if } b_i = 0.\\
\end{array} \right.
\end{equation}

Then $W_{\text{fed}}^{\calRO}$ follows a uniform distribution over $[\widetilde a^1 - \delta, \widetilde a^1 + \delta]\times [\widetilde a^2 - \delta, \widetilde a^2 + \delta]\cdots\times [\widetilde a^{m} - \delta, \widetilde a^{m} + \delta]$, where $\widetilde a^i = \frac{1}{K} \sum_{k = 1}^K a_k^i$, and $m$ represents the dimension of $W_k^{\calRO}$.

Then we have that
\begin{align*}
    \text{TV}(P^{\calRO}_{\text{fed}} || P^{\calD}_{\text{fed}}) 
    &= \int_{[\widetilde a^1 - \delta, \widetilde a^1 + \delta]}\int_{[\widetilde a^2 - \delta, \widetilde a^2 + \delta]}\cdots\int_{[\widetilde a^{m} - \delta, \widetilde a^{m} + \delta]} \left(\left(\frac{1}{2\delta}\right)^{m} - \prod_{i = 1}^{m}\left(\frac{\rho_i}{2\delta}\right)\right) dw_1 dw_2 \cdots d{w_{m}}\\
    & = \left(\left(\frac{1}{2\delta}\right)^{m} - \prod_{i = 1}^{m}\left(\frac{\rho_i}{2\delta}\right)\right)\cdot (2\delta)^{m}.
\end{align*}
From \pref{lem: eps_u_upper_bound}, we have that
\begin{align*}
    \epsilon_{u} \le C_4\cdot {\text{TV}}(P^{\calRO}_{\text{fed}} || P^{\calD}_{\text{fed}} ).
\end{align*}
Therefore, 
\begin{align*}
    \epsilon_u \le C_4\cdot \left(\left(\frac{1}{2\delta}\right)^{m} - \prod_{i = 1}^{m}\left(\frac{\rho_i}{2\delta}\right)\right)\cdot (2\delta)^{m}.
\end{align*}
\end{proof}

\begin{lem}\label{lem: efficiency_reduction_and_tvd_app_02}
For compression mechanism, the efficiency reduction $\epsilon_e\le C_5\cdot \left(1 - \rho^m\right)$.
\end{lem}
\begin{proof}

Assume that $W_k^{\calRO}$ follows a uniform distribution over $[a_k^1 - \delta, a_k^1 + \delta]\times [a_k^2 - \delta, a_k^2 + \delta]\cdots\times [a_k^{m} - \delta, a_k^{m} + \delta]$, where $m$ represents the dimension of $W_k^{\calRO}$. Besides,

\begin{equation}
W_k^{\calD}(i) =\left\{
\begin{array}{cl}
 W_k^{\calRO}(i)/\rho_i &  \text{if } b_i = 1,\\
0,  &  \text{if } b_i = 0.\\
\end{array} \right.
\end{equation}

From \pref{lem: eps_e_upper_bound_app},

\begin{align*}
    \epsilon_{e} &\le C_5\cdot\frac{1}{K}\sum_{k = 1}^K {\text{TV}}(P^{\calRO}_{k} || P^{\calD}_{k} )\\
    & = C_5\cdot\left(\left(\frac{1}{2\delta}\right)^{m} - \prod_{i = 1}^{m}\left(\frac{\rho_i}{2\delta}\right)\right)\cdot (2\delta)^{m}\\
    & = C_5\cdot \left(1 - \rho^m\right).
\end{align*}

\end{proof}

\subsection{Optimal Compression Probability}

\begin{figure}[h!]
\begin{framed}
\textbf{Optimal Compression Probability}
\begin{align*}
\begin{array}{r@{\quad}l@{}l@{\quad}l}
\quad\min\limits_{n} & \frac{1}{K}\sum_{k=1}^K  \left(\eta_{u,k}\cdot \widetilde \epsilon_{u,k}(\rho) +  \eta_{e,k}\cdot \widetilde\epsilon_{e,k}(\rho)\right)\\
\text{s.t.,} & \widetilde\epsilon_p (\rho)\le\epsilon.\\
\end{array}
\end{align*}

\end{framed}
\caption{Secure FL with Secret Sharing Mechanism}
\label{fig: Secure FL with Secret Sharing Mechanism_app_02}
\end{figure}

The following lemma provides an upper bound for privacy leakage, which could be further used to solve the \textit{relaxed optimization problem}. 
\begin{lem}\label{lem: privacy_leakge_paillier_upper_bound_app_03}
For the compression mechanism, the privacy leakage is upper bounded by 
\begin{align}\label{eq: privacy_bound_app_05}
    \epsilon_p \le 2C_1 - C_2\cdot\left(1 - \rho^m\right).
\end{align}
\end{lem}
\begin{proof}

Then we have that
\begin{align*}
    \text{TV}(P^{\calRO}_k || P^{\calD}_k) 
    &= \int_{[a_k^1 - \delta, a_k^1 + \delta]}\int_{[a_k^2 - \delta, a_k^2 + \delta]}\cdots\int_{[a_k^{m} - \delta, a_k^{m} + \delta]} \left(\left(\frac{1}{2\delta}\right)^{m} - \prod_{i = 1}^{m}\left(\frac{\rho_i}{2\delta}\right)\right) dw_1 dw_2 \cdots d{w_{m}}\\
    & = \left(\left(\frac{1}{2\delta}\right)^{m} - \prod_{i = 1}^{m}\left(\frac{\rho_i}{2\delta}\right)\right)\cdot (2\delta)^{m}.
\end{align*}

From \pref{thm: privacy_leakage_upper_bound_app} we have that

\begin{align*}
    \epsilon_p &\le 2C_1 - C_2\cdot\frac{1}{K}\sum_{k = 1}^K {\text{TV}}(P_k^{\calRO} || P^{\calD}_k)\\
    &\le 2C_1 - C_2\cdot\left(\left(\frac{1}{2\delta}\right)^{m} - \prod_{i = 1}^{m}\left(\frac{\rho_i}{2\delta}\right)\right)\cdot (2\delta)^{m}\\
    & = 2C_1 - C_2\cdot\left(1 - \rho^m\right).
\end{align*}


\end{proof}

\begin{thm}
The optimal value for the optimization problem (\pref{eq: relaxed_opt}) of the protector when the compression probability is set as
\begin{align*}
    \rho^* = \left(1 - \frac{2C_1 - \epsilon}{C_2}\right)^{\frac{1}{m}}.
\end{align*}
\end{thm}

\begin{proof}

\begin{align*}
    \widetilde\epsilon_p (\rho) = 2C_1 - C_2\cdot\left(1 - \rho^m\right)\le\epsilon.
\end{align*}

The second-order derivative
\begin{align*}
  \frac{\partial ^{2} \widetilde U}{\partial {\rho}^{2}}  = 0,
\end{align*}

and the first-order derivative

\begin{align*}
    \frac{\partial \widetilde U}{\partial {\rho}}  < 0.
\end{align*}

The optimal value for the optimization problem (\pref{eq: relaxed_opt}) of the protector when the compression probability is
\begin{align*}
    \rho^* = \left(1 - \frac{2C_1 - \epsilon}{C_2}\right)^{\frac{1}{m}}.
\end{align*}

\end{proof}

\end{document}